\documentclass[11pt,a4paper]{article}

\usepackage[margin=1.2in]{geometry}
\usepackage{amsmath,amssymb,amsthm}
\usepackage{mathtools}
\usepackage{graphicx}
\usepackage{xcolor}
\usepackage{hyperref}
\usepackage{booktabs}
\usepackage{microtype}
\usepackage{natbib}
\usepackage{algorithm}
\usepackage{algpseudocode}
\usepackage{caption}
\usepackage{subcaption}
\usepackage{enumitem}
\usepackage{thmtools}
\usepackage{multirow}
\usepackage{tcolorbox}
\usepackage{hyperref}

\hypersetup{colorlinks=true, linkcolor=blue!70!black, citecolor=green!50!black,
            urlcolor=blue!70!black}

\newtheorem{theorem}{Theorem}[section]
\newtheorem{lemma}[theorem]{Lemma}
\newtheorem{proposition}[theorem]{Proposition}
\newtheorem{corollary}[theorem]{Corollary}

\definecolor{theoblue}{RGB}{37,99,235}
\definecolor{empgreen}{RGB}{5,150,105}
\definecolor{warnorange}{RGB}{217,119,6}

\title{\textbf{Loss-Parameterized Fisher Width Along Learning Trajectories}}

\author{Vu Khac Ky\\[4pt]
\small Department of Mathematics, FPT University, Vietnam\\
\small \texttt{kyvk2@fe.edu.vn}}

\begin{document}
\maketitle

\begin{abstract}
Fisher width measures the Gaussian width of a probe set after
deformation by the local Fisher geometry. We study its evolution along
learning trajectories and ask when training loss can serve as an
effective coordinate for this quantity.

We first derive an exact trace--shape factorization and a deterministic
stability bound for fixed compact probes. In a population
Gaussian-teacher logistic model, the teacher-aligned state is extremal
on every loss level below $\log 2$: it has minimal parameter norm and
maximizes both Fisher trace and Euclidean-ball Fisher width. We then
show that population gradient flow asymptotically selects this branch,
with explicit rates for the aligned and orthogonal coordinates. This
yields, for $d\geq2$,
\[
    \frac{w_F(B_2^d;\theta(t))}
         {\sqrt{L(\theta(t))}}
    \longrightarrow
    \frac{\sqrt6}{\pi}\mathbb E[\chi_{d-1}].
\]

Controlled full-Fisher experiments support the matched-loss branch and
the population predictions. In a nonlinear MLP with a diagonal
model-Fisher approximation, GD and SGD remain close at matched loss,
whereas Adam follows a substantially displaced branch; the fixed probes
tested retain highly similar temporal shapes. These results support a
branchwise, rather than universal, loss parametrization of Fisher width.
\end{abstract}

\section{Introduction}
\label{sec:introduction}

\subsection{Fisher Width Along Learning Trajectories}

The Fisher information matrix defines the local metric of a statistical
model. For a parametric family
$\{p_\theta:\theta\in\Theta\subseteq\mathbb R^d\}$, let
\[
    G(\theta)
    =
    \mathbb E\!\left[
        \nabla_\theta\log p_\theta(X)
        \nabla_\theta\log p_\theta(X)^\top
    \right].
\]
For a nonempty compact probe set $T\subset\mathbb R^d$, its Fisher width
at $\theta$ is
\[
    w_F(T;\theta)
    =
    w\!\left(G(\theta)^{1/2}T\right),
\]
where
\[
    w(A)
    =
    \mathbb E_{g\sim N(0,I_d)}
    \sup_{v\in A}\langle g,v\rangle
\]
is the Gaussian width of $A$ \citep{Ky2026a}. Thus Fisher width measures
the size of a probe after deformation by the local Fisher geometry.

Previous uses of this quantity fix the base point $\theta$. Here the
base point moves along a learning trajectory,
\[
    \theta_0,\theta_1,\ldots,
\]
so Fisher width becomes the dynamical observable
\[
    t\longmapsto w_F(T;\theta_t).
\]
Both the scale and orientation of $G(\theta_t)$ can change during
training, and there is no reason for this observable to have a simple
dependence on iteration time.

We instead ask whether training loss can serve as an approximate
coordinate for Fisher width along selected trajectories:
\[
    w_F(T;\theta_t)
    \approx
    \Psi\!\left(L(\theta_t)\right).
\]
This is weaker than a global identity
\[
    w_F(T;\theta)
    =
    \Psi(L(\theta))
    \qquad
    \text{for all }\theta,
\]
which would require every equal-loss parameter to have the same Fisher
width. The controlled model studied below shows that such a global
identity is false. The useful formulation is branchwise: optimization
may select a region of parameter space on which loss becomes an
effective coordinate for a Fisher-geometric observable.

The aim is therefore not to reduce the full learning dynamics to one
scalar variable. We ask a narrower set of questions: when does a
loss-parametrized Fisher-width branch appear, what geometric and
dynamical mechanisms can produce it, and which parts of this
description persist when the controlled assumptions are relaxed?

\subsection{Matched-Loss Fisher-Width Branches}
\label{sec:intro-phenomenon}

Consider trajectories $\{\theta_t^{(a)}\}$ produced by different
optimizers, initializations, or stochastic perturbations. Rather than
comparing them at the same iteration number, we compare checkpoints at
the same loss. On a shared observed loss interval $\mathcal I$, the
resulting phase portraits may approximately satisfy
\begin{equation}
    w_F(T;\theta_t^{(a)})
    \approx
    a_T\,
    \Psi_b\!\left(L(\theta_t^{(a)})\right),
    \qquad
    L(\theta_t^{(a)})\in\mathcal I.
    \label{eq:intro-branch}
\end{equation}
Here $a_T$ is probe-dependent, while the branch $\Psi_b$ may depend on
the family of trajectories being compared. We use the term
\emph{Fisher-width equation of state} only as a descriptive shorthand
for such a branch relation; no thermodynamic interpretation or
universal loss-only law is assumed.

The empirical starting point is a controlled Gaussian-teacher logistic
regression problem. Full-batch gradient descent, damped natural
gradient descent, noisy gradient descent, and Adam follow closely
related full-Fisher-width curves when compared only over their shared
observed loss intervals. The matched-loss curve is also stable across
the small random initializations considered. These observations suggest
that, in this setting, iteration time contains optimizer-specific
information that is largely removed when the trajectories are
reparametrized by loss.

The same behaviour is not optimizer-independent outside the controlled
model. In a two-hidden-layer MLP, GD and SGD remain close at matched
loss under a diagonal approximation to the model Fisher, whereas Adam
follows a substantially displaced branch. At the same time, the fixed
ball, random-subspace, sparse, and ellipsoidal probes considered in the
experiment have strongly correlated temporal shapes. Thus the nonlinear
experiment identifies optimizer dependence more clearly than probe
dependence within this fixed family. These observations motivate a
branchwise description without suggesting that all optimizers,
architectures, or probe sets share a common law.

A first geometric mechanism follows from the exact factorization
\[
    G(\theta)
    =
    s(\theta)\,\overline G(\theta),
    \qquad
    s(\theta)=\operatorname{Tr}G(\theta),
    \qquad
    \operatorname{Tr}\overline G(\theta)=1.
\]
Positive homogeneity of Gaussian width gives
\begin{equation}
    w_F(T;\theta)
    =
    \sqrt{\operatorname{Tr}G(\theta)}\,
    w\!\left(\overline G(\theta)^{1/2}T\right).
    \label{eq:intro-trace-shape}
\end{equation}
We prove a deterministic perturbation bound showing that small operator
drift of the normalized square-root Fisher matrix controls the second
factor for every fixed compact probe. The bound retains an explicit
probe-dependent sensitivity factor. Hence slowly varying normalized
Fisher geometry can explain similar probe evolution, but it does not
by itself imply that Fisher width is a function of loss.

The Gaussian-logistic population model provides the additional
structure needed to connect geometry with loss. Let
\[
    X\sim N(0,I_d),
    \qquad
    Y=\operatorname{sign}\langle u_\star,X\rangle,
    \qquad
    \|u_\star\|_2=1.
\]
Writing
\[
    \theta
    =
    \alpha u_\star+\beta w,
    \qquad
    w\perp u_\star,
    \qquad
    r=\|\theta\|_2,
\]
the population Fisher matrix depends only on $r$ and has the exact
one-spike form
\begin{equation}
    G(ru)
    =
    b(r)(I-P_u)+a(r)P_u.
    \label{eq:intro-one-spike}
\end{equation}
The supervised loss, in contrast, depends separately on the aligned and
orthogonal coordinates $(\alpha,\beta)$.

This separation yields a global extremal property. Define the aligned
population loss
\[
    L_\star(r)
    :=
    L(ru_\star).
\]
For every state with $L(\theta)<\log 2$,
\begin{equation}
    L(\alpha,\beta)-L_\star(r)
    \geq
    \frac12 b(r)\beta^2.
    \label{eq:intro-global-penalty}
\end{equation}
If
\[
    r_L
    :=
    L_\star^{-1}(L(\theta)),
\]
then $\|\theta\|_2\geq r_L$, with strict inequality away from the
aligned branch. Since the Fisher trace and the Euclidean-ball Fisher
width both decrease with radius,
\begin{equation}
    \operatorname{Tr}G(\theta)
    \leq
    \operatorname{Tr}G(r_Lu_\star),
    \qquad
    w_F(B_2^d;\theta)
    \leq
    w_F(B_2^d;r_Lu_\star).
    \label{eq:intro-envelope}
\end{equation}
Thus the aligned branch is an upper envelope for these Fisher
observables on every low-loss level set. This both identifies a
distinguished loss-parametrized branch and rules out a global
loss-only equality away from it.

The next question is whether learning dynamics actually select this
extremal branch. Population gradient flow satisfies
\begin{equation}
    \dot\alpha
    =
    \mathbb E\!\left[
        Z_1\sigma(-\alpha Z_1-\beta Z_2)
    \right]
    >0,
    \qquad
    \dot\beta
    =
    -\beta\,b(r),
    \label{eq:intro-gradient-flow}
\end{equation}
where $Z_1\sim|N(0,1)|$ and $Z_2\sim N(0,1)$ are independent. We prove
that
\begin{equation}
    \alpha(t)\longrightarrow\infty,
    \qquad
    \beta(t)\longrightarrow0,
    \qquad
    \frac{\alpha(t)}{\|\theta(t)\|_2}\longrightarrow1.
    \label{eq:intro-alignment}
\end{equation}
Moreover,
\begin{equation}
    \frac{\alpha(t)^3}{t}
    \longrightarrow
    \frac{\pi^2}{2\sqrt{2\pi}},
    \label{eq:intro-alpha-rate}
\end{equation}
and, when $\beta(0)\neq0$,
\begin{equation}
    \frac{\log|\beta(t)|}{\alpha(t)^2}
    \longrightarrow
    -\frac{3}{\pi^2}.
    \label{eq:intro-beta-rate}
\end{equation}
Hence population gradient flow asymptotically selects the aligned
branch rather than merely contracting the orthogonal coordinate.

Near this branch, the matched-loss Fisher-width residual is
$O(|\beta|)$ for a general fixed compact probe and $O(\beta^2)$ for
rotation-invariant probes. Combining this stability with dynamical
selection and the aligned low-loss asymptotics gives, for $d\geq2$,
\begin{equation}
    \frac{
        w_F(B_2^d;\theta_t)
    }{
        \sqrt{L(\theta_t)}
    }
    \longrightarrow
    \frac{\sqrt6}{\pi}\,
    \mathbb E[\chi_{d-1}]
    \qquad
    \text{under population gradient flow}.
    \label{eq:intro-dynamic-eos}
\end{equation}
Thus the late loss-parametrized Fisher-width law holds along the actual
population trajectory, not only after restricting the parameter
artificially to the aligned ray.

The numerical experiments later in the paper serve two distinct
purposes. Controlled population calculations test the level-set
envelope, the asymptotic alignment rates, and the dynamic Fisher-width
law. Finite-sample and MLP experiments then examine which parts of the
description remain visible once sampling, nonlinear parametrization,
optimizer dependence, and approximate Fisher geometry are introduced.

\subsection{Contributions and Scope}
\label{sec:intro-contributions}

The contributions are organized around four points.

\begin{enumerate}

    \item \textbf{Matched-loss Fisher-width branches and a trace--shape
    mechanism.}
    We study Fisher width as a dynamical observable and compare
    optimization trajectories at matched loss rather than matched
    iteration. In the controlled logistic experiment, several
    optimizers and small random initializations produce closely agreeing
    full-Fisher-width branches over their shared observed loss ranges.
    In the MLP experiment, GD and SGD remain close under a diagonal
    model-Fisher geometry, while Adam exhibits a substantial
    optimizer-dependent displacement; the fixed probes considered
    nevertheless retain strongly correlated temporal shapes.

    Independently of these experiments, we prove the exact factorization
    \[
        w_F(T;\theta)
        =
        \sqrt{\operatorname{Tr}G(\theta)}\,
        w\!\left(\overline G(\theta)^{1/2}T\right)
    \]
    and a deterministic perturbation bound for every fixed compact
    probe. The result separates Fisher scale from normalized Fisher
    shape and makes the probe-dependent sensitivity explicit.

    \item \textbf{Global extremality of the aligned branch.}
    In the Gaussian-teacher logistic population model, the Fisher
    matrix has an exact one-spike structure. The separation between
    radial Fisher geometry and alignment-dependent loss yields the
    global misalignment penalty
    \[
        L(\alpha,\beta)-L_\star(\|\theta\|_2)
        \geq
        \frac12b(\|\theta\|_2)\beta^2.
    \]
    Consequently, among parameters with the same loss below $\log2$,
    the aligned state has the smallest norm and maximizes both Fisher
    trace and Euclidean-ball Fisher width. The aligned
    loss-parametrized relation is therefore a level-set envelope rather
    than a global identity.

    \item \textbf{Dynamical selection and an asymptotic Fisher-width
    equation of state.}
    For population gradient flow, we derive the exact two-coordinate
    dynamics and prove
    \[
        \alpha(t)\to\infty,
        \qquad
        \beta(t)\to0,
    \]
    together with the asymptotic rates
    \[
        \alpha(t)^3
        \sim
        \frac{\pi^2}{2\sqrt{2\pi}}\,t,
        \qquad
        \log|\beta(t)|
        \sim
        -\frac{3}{\pi^2}\alpha(t)^2
    \]
    when $\beta(0)\neq0$. Combining this dynamical selection with the
    Fisher eigenvalue asymptotics yields
    \[
        w_F(B_2^d;\theta_t)
        \sim
        \frac{\sqrt6}{\pi}\,
        \mathbb E[\chi_{d-1}]
        \sqrt{L(\theta_t)}
    \]
    for $d\geq2$. The loss-parametrized law therefore holds
    asymptotically along the population learning trajectory itself.

    \item \textbf{Finite-sample control and empirical boundaries.}
    We give an exact empirical Fisher-trace identity and pointwise
    concentration bounds under independent Gaussian evaluation data.
    For finitely many recorded checkpoints, the bounds extend by a
    union argument. Controlled numerical experiments test the global
    envelope, the alignment rates, and the late Fisher-width law.

    We then leave the Gaussian linear setting and study a two-hidden-layer
    MLP using a diagonal approximation to the model Fisher. The 
    nonlinear experiment separates two effects: the fixed probes tested
    retain very similar temporal shapes, whereas Adam follows a clearly
    displaced matched-loss branch relative to GD. These results are
    empirical boundary tests and are not used to infer a full-Fisher
    theorem for neural networks.

\end{enumerate}

The theoretical scope is specific. The global envelope and dynamical
selection results concern population logistic regression with isotropic
Gaussian covariates and a noiseless linear teacher. The near-aligned
bounds are uniform on compact aligned-coordinate intervals away from
the singular endpoint of the inverse loss parametrization. The
finite-sample concentration result is pointwise in the evaluated
parameter, or simultaneous over a finite set of checkpoints under an
independent evaluation sample; it is not a trajectory-uniform empirical
process theorem. The nonlinear experiment uses a diagonal approximation
to the model Fisher and does not imply corresponding behaviour for the
full Fisher matrix or for neural networks in general. In particular,
we do not claim that loss globally determines Fisher width, that all
optimizers select the same branch, or that arbitrary probe sets have
the same dynamics.

\subsection{Related Work}
\label{sec:related-work}

\paragraph{Information geometry and natural-gradient methods.}
The Fisher information is the canonical Riemannian metric of a
statistical model and provides the basis of natural-gradient
optimization \citep{Amari1998,AmariNagaoka2000}. For neural networks,
practical curvature approximations include extended Gauss--Newton
methods, Kronecker-factored approximations, and related structured
methods
\citep{PascanuBengio2014,MartensGrosse2015,GrosseMartens2016,
BotevRitterBarber2017,Martens2020}.
The model Fisher and the empirical Fisher should be distinguished:
the latter is not generally an interchangeable curvature matrix
\citep{Kunstner2019}. Most of this literature uses Fisher or related
matrices to construct optimization directions. Here the Fisher matrix
instead defines a geometric observable whose evolution is measured
along the learning trajectory.

\paragraph{Fisher, Hessian, and tangent geometry during training.}
The spectrum of the Fisher information has been studied in random and
wide neural networks \citep{PenningtonWorah2018,Karakida2021}, and
early-training Fisher growth has been related empirically to
optimization and generalization \citep{Jastrzebski2021}. A parallel
literature studies the evolution of Hessian geometry, including
outlying eigenvalues, sharpening, edge-of-stability behaviour, and
large-learning-rate regimes
\citep{Ghorbani2019,Cohen2021,Lewkowycz2020}.
Neural-tangent-kernel and lazy-training analyses identify regimes in
which tangent geometry remains fixed or changes weakly
\citep{Jacot2018,ChizatBach2019,Lee2019}. These matrices can be related
in special models, but the population Fisher, empirical Fisher, loss
Hessian, generalized Gauss--Newton matrix, and NTK are not identical.
The present work focuses specifically on Gaussian width after
deformation by the Fisher metric.

\paragraph{Gaussian width and Fisher width.}
Gaussian width is a standard complexity measure in high-dimensional
probability and Gaussian-process geometry
\citep{Gordon1988,LedouxTalagrand1991,Vershynin2018}. Gaussian and
Rademacher complexities also play a central role in statistical
learning bounds \citep{BartlettMendelson2002}, while Gaussian widths
and related conic quantities characterize random inverse problems and
convex recovery \citep{Chandrasekaran2012,Amelunxen2014}. Fisher width
applies this construction after the local Fisher deformation,
\[
    w_F(T;\theta)
    =
    w\!\left(G(\theta)^{1/2}T\right),
\]
and was introduced as a local complexity measure on statistical
manifolds \citep{Ky2026a}. Related work considers primal and
inverse-Fisher widths and their learning and recovery roles
\citep{Ky2026b}. Those results are formulated at a fixed base point.
The present paper instead studies the quantity along a moving
optimization trajectory and asks when its evolution admits a
loss-based parametrization.

\paragraph{Reduced descriptions and teacher--student dynamics.}
Teacher--student analyses often reduce high-dimensional learning
dynamics to a small collection of overlaps and order parameters
\citep{SaadSolla1995,Goldt2019}. Mean-field limits provide another form
of reduced description in which wide-network training is represented
by the evolution of a parameter distribution
\citep{MeiMontanariNguyen2018}. Strong geometric organization can also
appear late in training, as illustrated by neural collapse
\citep{Papyan2020}. Our two-coordinate Gaussian-teacher reduction is
in this broad tradition, but the object of interest is different: we
study whether one Fisher-geometric observable becomes parametrizable
by loss and whether the dynamics select the corresponding extremal
branch. We do not infer that the complete learning dynamics are
one-dimensional.

\paragraph{Implicit bias and optimizer dependence.}
Different optimization algorithms can select different
parameter-space geometries even when they attain similar loss. Adam
uses a history-dependent coordinatewise preconditioner
\citep{KingmaBa2015}, and differences between adaptive and
non-adaptive methods have been documented in generalization and
solution selection \citep{Wilson2017}. More generally, optimization
geometry can determine implicit bias \citep{Gunasekar2018}. For
logistic-type losses on finite separable datasets, gradient descent
approaches maximum-margin directions despite the absence of a finite
minimizer \citep{Soudry2018,Nacson2019}, with related analyses of risk
and parameter convergence \citep{JiTelgarsky2019}. The aligned branch
studied here is different: it arises from population Gaussian symmetry
rather than from a finite-sample maximum-margin problem. Our dynamical
result concerns the rate at which population gradient flow selects
this branch and the Fisher-width asymptotics induced by that selection.

\paragraph{Position of the present work.}
The ingredients above have largely been studied separately:
information geometry provides the metric, Gaussian width provides the
set complexity, and teacher--student and implicit-bias analyses provide
reduced descriptions of learning trajectories. The present paper asks
how these ingredients interact when Fisher width is evaluated during
optimization. The empirical matched-loss phenomenon motivates a
trace--shape decomposition; the controlled Gaussian-logistic model
then yields a global level-set envelope, proves asymptotic dynamical
selection of that envelope, and gives an explicit late-loss
Fisher-width law. Finite-sample and nonlinear experiments are used to
test which parts of this branchwise description persist when the
controlled assumptions are relaxed.

\section{Dynamic Fisher Width and Matched-Loss Agreement}
\label{sec:empirical}

We begin with the empirical phenomenon that motivates the theory. The
question is not whether loss determines Fisher width throughout
parameter space, but whether trajectories generated within the same
learning problem exhibit similar Fisher width when compared at the
same loss.

For a trajectory $\{\theta_t^{(a)}\}$, we record the phase portrait
\[
    \left(
        \widehat L(\theta_t^{(a)}),
        \widehat w_F(B_2^d;\theta_t^{(a)})
    \right).
\]
If two trajectories follow similar curves after reparametrization by
loss, we regard them as following the same observed matched-loss
branch. This section studies the phenomenon in a controlled logistic
problem using the full Fisher matrix. Experiments outside this setting
are deferred to Section~\ref{sec:beyond-controlled}.

\subsection{Experimental Setup and Matched-Loss Protocol}
\label{sec:setup}

We consider binary logistic regression with
\[
    x_i\sim N(0,I_d),
    \qquad
    y_i=\operatorname{sign}\langle u_\star,x_i\rangle,
\]
where
\[
    d=30,
    \qquad
    n=200,
    \qquad
    \|u_\star\|_2=1.
\]
The empirical logistic loss is
\[
    \widehat L(\theta)
    =
    \frac1n
    \sum_{i=1}^n
    \log\!\left(
        1+\exp(-y_i\langle\theta,x_i\rangle)
    \right).
\]

The teacher and dataset are fixed across all runs. For each of six
initializations,
\[
    \theta_0\sim N(0,0.05^2 I_d),
\]
the same initial parameter is used for all optimizers. We compare
full-batch gradient descent (GD), damped natural gradient descent
(NGD), gradient descent with additive Gaussian perturbations
(Noisy GD), and Adam. GD and NGD use step size $\eta=0.05$, with NGD
update
\[
    \theta_{t+1}
    =
    \theta_t
    -
    \eta
    \left(
        \widehat G(\theta_t)+10^{-4}I_d
    \right)^{-1}
    \nabla\widehat L(\theta_t).
\]
Noisy GD uses
\[
    \theta_{t+1}
    =
    \theta_t
    -
    0.05\,\nabla\widehat L(\theta_t)
    +
    0.005\,\xi_t,
    \qquad
    \xi_t\sim N(0,I_d),
\]
with independent perturbations across iterations. Adam uses learning
rate $10^{-3}$ and the standard parameters
$\beta_1=0.9$, $\beta_2=0.999$, and $\varepsilon=10^{-8}$.
All trajectories are run for $600$ iterations.

For this model, we use the empirical average of the conditional model
Fisher,
\[
    \widehat G(\theta)
    =
    \frac1n
    \sum_{i=1}^n
    v(\langle\theta,x_i\rangle)x_ix_i^\top,
    \qquad
    v(z)=\sigma(z)(1-\sigma(z)).
\]
Since $d=30$, the full $d\times d$ matrix is formed directly; no
diagonal or trace approximation is used. For the Euclidean-ball probe,
\[
    \widehat w_F(B_2^d;\theta)
    =
    \mathbb E_{g\sim N(0,I_d)}
    \left\|
        \widehat G(\theta)^{1/2}g
    \right\|_2.
\]
We estimate this expectation using $B=1200$ Gaussian draws. A single
fixed bank of Gaussian draws is reused across checkpoints, optimizers,
and initializations, so Monte Carlo variation does not contribute to
differences between the recorded phase portraits.

Matched-loss comparisons are restricted to losses observed by both
trajectories. For trajectories $a$ and $b$, define
\begin{equation}
    \mathcal I_{a,b}
    =
    \left[
        \max\{L_{\min}^{(a)},L_{\min}^{(b)}\},
        \min\{L_{\max}^{(a)},L_{\max}^{(b)}\}
    \right].
    \label{eq:shared-loss-range}
\end{equation}
No extrapolation outside $\mathcal I_{a,b}$ is used. The recorded
phase portraits are sorted by loss and interpolated piecewise linearly
on a common grid $\mathcal G\subset\mathcal I_{a,b}$. We write
\[
    Q^{(a)}(\ell)
    =
    \widehat w_F
    \bigl(
        B_2^d;\theta^{(a)}(\ell)
    \bigr)
\]
for the resulting interpolated width.

Using GD as the reference trajectory, let
$\widehat\Psi_{\mathrm{GD}}(\ell)$ denote its interpolated matched-loss
curve and define
\[
    \rho_a(\ell)
    =
    \frac{
        Q^{(a)}(\ell)
    }{
        \widehat\Psi_{\mathrm{GD}}(\ell)
    }.
\]
We summarize the multiplicative displacement over the common loss
range by the geometric mean
\begin{equation}
    \overline\rho_a
    =
    \exp\left\{
        \frac1{|\mathcal G|}
        \sum_{\ell\in\mathcal G}
        \log\rho_a(\ell)
    \right\}.
    \label{eq:matched-loss-ratio}
\end{equation}
A value close to one indicates agreement with the GD reference over
the shared observed loss range. It does not imply equality of the
parameter trajectories or of the Fisher matrices themselves.

\subsection{Matched-Loss Agreement Across Optimizers}
\label{sec:optimizer-comparison}

The full-Fisher experiment shows close matched-loss agreement
among the four methods. Across the six initializations, damped NGD
gives
\[
    \overline\rho_{\mathrm{NGD}}
    =
    1.009\pm0.002,
\]
while Noisy GD gives
\[
    \overline\rho_{\mathrm{noise}}
    =
    0.998\pm0.001.
\]
Adam also remains close to the GD reference on the shared observed
loss range:
\[
    \overline\rho_{\mathrm{Adam}}
    =
    0.995\pm0.003.
\]
Here and below, the reported variation is the sample standard
deviation across the six initializations.

\begin{table}[htbp]
\centering
\caption{Matched-loss full-Fisher-width ratios in the controlled
logistic experiment. Each optimizer is compared with GD only over
their shared observed loss range. Values are mean $\pm$ sample
standard deviation across six initializations.}
\label{tab:optimizer-ratios}
\begin{tabular}{lcl}
\toprule
Optimizer & $\overline\rho$ & Observation \\
\midrule
GD       & $1.000$         & reference \\
NGD      & $1.009\pm0.002$ & close agreement \\
Noisy GD & $0.998\pm0.001$ & close agreement \\
Adam     & $0.995\pm0.003$ & close agreement \\
\bottomrule
\end{tabular}
\end{table}

Thus the four methods need not follow the same trajectory in parameter
space or in iteration time to produce similar Fisher widths at matched
loss. The agreement is not exact---NGD, for example, lies slightly
above the GD reference on average---but the observed displacement is
small relative to the scale of the width.

This experiment does not establish optimizer invariance. It shows only
that, for this controlled problem and initialization regime, the tested
optimizers occupy closely related branches in the
$(\widehat L,\widehat w_F)$ phase portrait. Section~\ref{sec:beyond-controlled}
gives a nonlinear example in which optimizer dependence is much more
pronounced.

\begin{figure}[t]
    \centering
    \includegraphics[width=\textwidth]{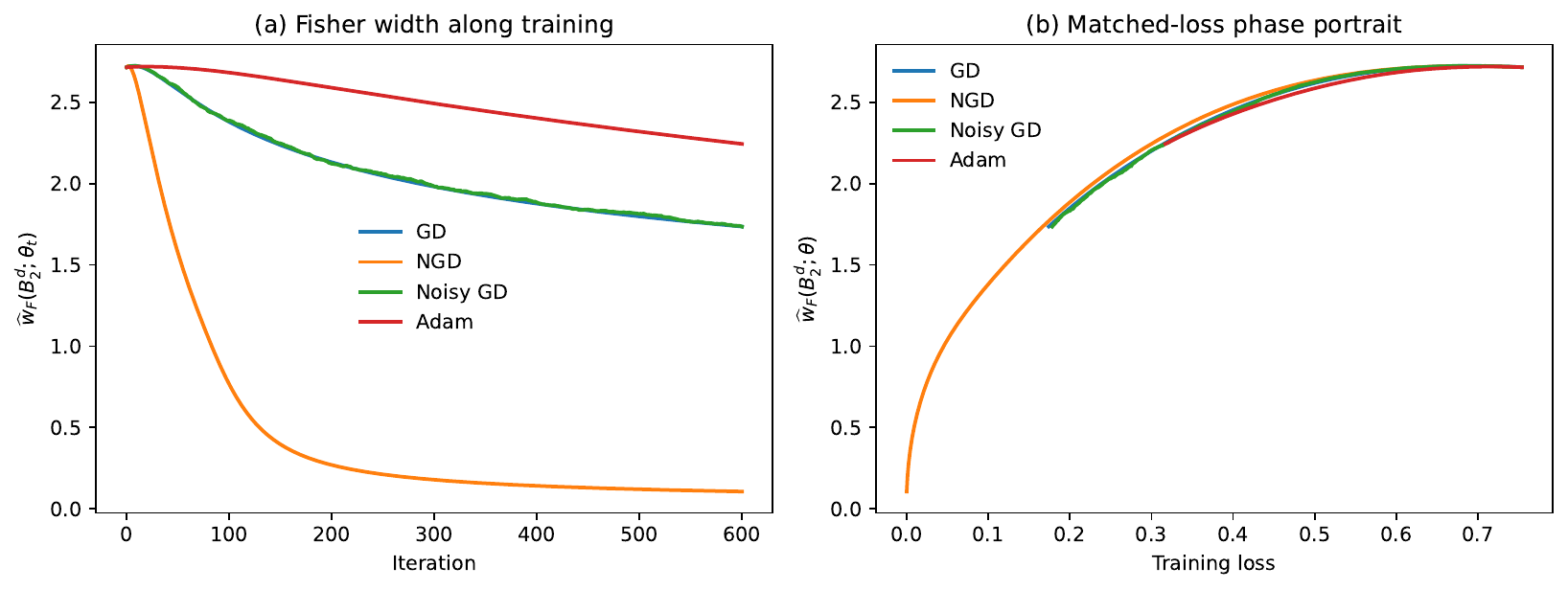}
    \caption{
    Controlled logistic experiment for one predetermined initialization.
    Left: full model-Fisher width along training.
    Right: the same trajectories reparametrized by training loss.
    The matched-loss curves are much more closely aligned than their
    iteration-time trajectories. Numerical summaries over all six
    initializations are reported in Table~\ref{tab:optimizer-ratios}.
    }
    \label{fig:logistic-matched-loss}
\end{figure}

Figure~\ref{fig:logistic-matched-loss} illustrates the same effect for
one predetermined initialization: the optimizer-dependent trajectories
are more closely aligned after reparametrization by loss.

\subsection{Stability Across Initializations}
\label{sec:initialization-stability}

We next isolate initialization variability. The dataset, objective,
optimizer, hyperparameters, and Gaussian width draws are fixed, and
only the initial parameter is changed. For the six GD trajectories,
let
\[
    \operatorname{CV}(\ell)
    =
    \frac{
        \operatorname{sd}_k Q^{(k)}(\ell)
    }{
        \operatorname{mean}_k Q^{(k)}(\ell)
    }
\]
on their common observed loss interval. Averaging over the common loss
grid gives
\[
    \overline{\operatorname{CV}}
    =
    0.00113.
\]
Thus, within the small-initialization regime considered here, the
full-Fisher width is highly reproducible after conditioning on loss.

This is an empirical statement about the tested initialization scale,
not a claim about arbitrary initial conditions. In particular,
equal-loss states need not have equal Fisher geometry. The population
analysis in Section~\ref{sec:gaussian-theory} will show explicitly that
misaligned states can have different norms and different Fisher widths
at the same loss.

\subsection{Branch Estimation and Scope}
\label{sec:nonparametric-eos}

We do not impose a parametric form on the empirical branch. The
reference curve
\[
    \widehat\Psi_{\mathrm{GD}}:
    \mathcal I
    \longrightarrow
    \mathbb R_+
\]
is defined directly by piecewise-linear interpolation of the observed
GD phase portrait. It is used only inside the observed loss range. In
particular, we do not infer a power law, log-linear law, or other
global functional form from the finite-range experiment.

The experiment also does not provide a finite-sample trajectory
theorem. The same fixed dataset is used both to generate the
optimization trajectories and to evaluate the sample Fisher matrix.
The role of this section is therefore exploratory: it identifies a
stable matched-loss phenomenon that the later theory will explain in a
controlled population model. Section~\ref{sec:finite-sample} separately
considers independent evaluation data and gives pointwise and
finite-checkpoint concentration guarantees.

The evidence from the controlled logistic problem can therefore be
summarized narrowly. Several optimization methods and small random
initializations produce closely related full-Fisher-width curves when
compared at matched loss, but the observation by itself does not imply
a global loss-only law. Section~\ref{sec:trace-shape} next separates
Fisher scale from Fisher shape and quantifies the stability of fixed
probes. Section~\ref{sec:gaussian-theory} then identifies the
level-set and dynamical mechanisms that generate a distinguished
loss-parametrized branch in the corresponding population model.

\section{Trace--Shape Factorization and Probe Stability}
\label{sec:trace-shape}

Matched-loss agreement concerns the scalar quantity
$w_F(T;\theta_t)$, but its variation can arise from two different
sources: the overall scale of the Fisher matrix and its normalized
shape. This section separates these effects. The resulting stability
bound is deterministic and applies to any fixed compact probe. It does
not assume a relation between Fisher trace and loss, and therefore does
not by itself produce a loss-parametrized branch.

\subsection{Trace--Shape Factorization}
\label{sec:trace-shape-factorization}

Let
\[
    G_t:=G(\theta_t),
    \qquad
    s_t:=\operatorname{Tr}G_t>0,
\]
and define the trace-normalized Fisher matrix
\[
    \overline G_t
    :=
    \frac{G_t}{s_t},
    \qquad
    \operatorname{Tr}\overline G_t=1.
\]
Since
\[
    G_t^{1/2}
    =
    \sqrt{s_t}\,\overline G_t^{1/2},
\]
positive homogeneity of Gaussian width gives the exact identity
\begin{equation}
    w_F(T;\theta_t)
    =
    \sqrt{s_t}\,
    w\!\left(\overline G_t^{1/2}T\right).
    \label{eq:exact-trace-shape-width}
\end{equation}
Equivalently,
\begin{equation}
    \frac{w_F(T;\theta_t)}
    {\sqrt{\operatorname{Tr}G_t}}
    =
    w\!\left(\overline G_t^{1/2}T\right).
    \label{eq:normalized-fisher-width}
\end{equation}

Thus $\sqrt{s_t}$ is a probe-independent Fisher-scale factor, whereas
\[
    w\!\left(\overline G_t^{1/2}T\right)
\]
contains the interaction between the normalized Fisher geometry and the
probe. In particular, similar Fisher traces do not imply similar
Fisher widths unless the normalized geometry is also controlled.

\subsection{Normalized Square-Root Fisher Shape}
\label{sec:normalized-shape-drift}

Let $\overline G_\star$ be a fixed positive semidefinite reference
matrix with
\[
    \operatorname{Tr}\overline G_\star=1.
\]
We measure normalized shape variation by
\begin{equation}
    \varepsilon_t
    :=
    \left\|
        \overline G_t^{1/2}
        -
        \overline G_\star^{1/2}
    \right\|_{\mathrm{op}}.
    \label{eq:sqrt-shape-drift}
\end{equation}

The square-root matrix is the natural object because Fisher width
depends directly on $G^{1/2}$. Without a positive spectral lower
bound, a perturbation estimate for
$\overline G_t-\overline G_\star$ does not translate at the same scale
into an estimate for their square roots. We therefore formulate the
stability result directly in terms of
$\overline G^{1/2}$.

\subsection{Probe-Stability Bound}
\label{sec:probe-stability-bound}

For a nonempty compact set $T\subset\mathbb R^d$, define
\[
    R_T
    :=
    \sup_{v\in T}\|v\|_2,
    \qquad
    \gamma_d
    :=
    \mathbb E\|g\|_2,
    \qquad
    g\sim N(0,I_d).
\]
Here
\[
    \gamma_d
    =
    \sqrt{2}\,
    \frac{\Gamma((d+1)/2)}{\Gamma(d/2)}
    \leq
    \sqrt d.
\]

\begin{theorem}[Gaussian-width perturbation and probe stability]
\label{thm:probe-stability}
Let $T\subset\mathbb R^d$ be nonempty and compact. For any linear maps
$A,B:\mathbb R^d\to\mathbb R^d$,
\begin{equation}
    |w(AT)-w(BT)|
    \leq
    \gamma_dR_T
    \|A-B\|_{\mathrm{op}}.
    \label{eq:linear-width-perturbation}
\end{equation}
Consequently, if
\[
    a_T
    :=
    w\!\left(\overline G_\star^{1/2}T\right),
\]
then
\begin{equation}
    \left|
        \frac{w_F(T;\theta_t)}{\sqrt{s_t}}
        -
        a_T
    \right|
    \leq
    \gamma_dR_T\varepsilon_t.
    \label{eq:absolute-normalized-probe-bound}
\end{equation}
If $a_T>0$, then
\begin{equation}
    \left|
        \frac{w_F(T;\theta_t)}
        {a_T\sqrt{s_t}}
        -1
    \right|
    \leq
    \frac{\gamma_dR_T}{a_T}\,
    \varepsilon_t.
    \label{eq:relative-probe-bound}
\end{equation}
\end{theorem}

The proof is given in Appendix~\ref{app:trace-shape-proofs}.

The relative estimate depends on the probe-sensitivity factor
\begin{equation}
    \chi_T
    :=
    \frac{\gamma_dR_T}{a_T}.
    \label{eq:probe-sensitivity}
\end{equation}
Hence small normalized Fisher-shape drift does not imply the same
relative accuracy for every probe. In particular, a probe with small
reference width $a_T$ can remain sensitive to changes in Fisher
orientation even when $\varepsilon_t$ is small.

The fixed-probe assumption is essential to the statement. An unbounded
cone must first be restricted to a compact set, for example
\[
    \operatorname{cone}(V)\cap B_2^d,
\]
and a sparse probe must include a norm constraint. A probe selected
adaptively from the current Fisher matrix, such as a time-dependent
leading eigenspace, is not covered because both the metric and the
probe then vary with $t$.

Theorem~\ref{thm:probe-stability} is a geometric stability result, not
an equation-of-state theorem. It shows that, for a fixed probe,
\[
    w_F(T;\theta_t)
    \approx
    a_T\sqrt{\operatorname{Tr}G_t}
\]
when the normalized square-root Fisher shape remains close to its
reference value. A loss parametrization requires additional information
linking Fisher scale and the portion of parameter space selected by the
learning dynamics. That additional structure is developed in
Section~\ref{sec:gaussian-theory}.

\subsection{Empirical Shape and Probe Diagnostic}
\label{sec:empirical-shape-diagnostics}

We now evaluate the trace--shape mechanism on the controlled logistic
problem of Section~\ref{sec:empirical}, where the full $30\times30$
empirical model Fisher matrix is available. This experiment is a
diagnostic of the deterministic bound above; it is not used to infer a
loss-only relation.

For one of the recorded GD trajectories, define at every checkpoint
\[
    \widehat G_t
    =
    \frac1n
    \sum_{i=1}^n
    v(\langle\theta_t,x_i\rangle)x_ix_i^\top,
    \qquad
    \widehat s_t
    =
    \operatorname{Tr}\widehat G_t,
\]
and
\[
    \overline{\widehat G}_t
    =
    \frac{\widehat G_t}{\widehat s_t}.
\]
We take as reference shape the average over the recorded checkpoints,
\[
    \overline G_\star
    :=
    \frac1{|\mathcal J|}
    \sum_{t\in\mathcal J}
    \overline{\widehat G}_t,
\]
where $\mathcal J$ denotes the checkpoints on this trajectory. The
corresponding square-root shape diagnostic is
\begin{equation}
    \widehat\varepsilon_t
    :=
    \left\|
        \overline{\widehat G}_t^{1/2}
        -
        \overline G_\star^{1/2}
    \right\|_{\mathrm{op}}.
    \label{eq:empirical-sqrt-shape-drift}
\end{equation}

We consider four fixed compact probes, all normalized to have Euclidean
radius one:
\[
    T_{\mathrm{ball}}
    =
    B_2^d,
\]
\[
    T_{\mathrm{sub}}
    =
    \operatorname{range}(Q)\cap B_2^d,
    \qquad
    Q\in\mathbb R^{d\times 8},
    \qquad
    Q^\top Q=I_8,
\]
\[
    T_{\mathrm{sparse}}
    =
    \left\{
        v:
        \|v\|_0\leq5,\;
        \|v\|_2\leq1
    \right\},
\]
and
\[
    T_{\mathrm{ell}}
    =
    D^{1/2}B_2^d,
    \qquad
    \max_j D_{jj}=1,
\]
where $Q$ and the diagonal ellipsoid $D$ are sampled once and then held
fixed over the trajectory. For the ellipsoid used here,
$D_{jj}\in\{0.05,1\}$, so its Euclidean radius is one. None of these
probes is selected adaptively from the current Fisher matrix.

The Gaussian widths are estimated using the same bank of $1200$
Gaussian draws used in the controlled experiment. Let
$\widehat a_T$ denote the Monte Carlo estimate of
\[
    a_T
    =
    w\!\left(\overline G_\star^{1/2}T\right),
\]
and define the corresponding empirical sensitivity factor
\[
    \widehat\chi_T
    :=
    \frac{\gamma_dR_T}{\widehat a_T}.
\]
Since all four probes have $R_T=1$, this reduces to
$\widehat\chi_T=\gamma_d/\widehat a_T$.

For each checkpoint we also compute
\begin{equation}
    \widehat\Delta_T(t)
    :=
    \left|
        \frac{
            \widehat w_F(T;\theta_t)
        }{
            \widehat a_T\sqrt{\widehat s_t}
        }
        -1
    \right|.
    \label{eq:empirical-probe-residual}
\end{equation}
The empirical counterpart of
Theorem~\ref{thm:probe-stability} is then compared with
\begin{equation}
    \widehat\chi_T\widehat\varepsilon_t.
    \label{eq:empirical-probe-bound}
\end{equation}
Because the widths and $\widehat a_T$ are Monte Carlo estimates,
\eqref{eq:empirical-probe-bound} is used here as a numerical benchmark;
the exact deterministic inequality is
\eqref{eq:relative-probe-bound}.

\begin{table}[htbp]
\centering
\caption{Trace--shape diagnostics for one GD trajectory in the
controlled logistic experiment. The quantities $\widehat a_T$ and
$\widehat\Delta_T$ are estimated with $1200$ shared Gaussian draws.
The last column is the corresponding empirical benchmark from the
probe-stability bound.}
\label{tab:shape-probe-diagnostics}
\begin{tabular}{lcccc}
\toprule
Probe
& $\widehat a_T$
& $\widehat\chi_T$
& $\max_t\widehat\Delta_T(t)$
& $\max_t\widehat\chi_T\widehat\varepsilon_t$ \\
\midrule
$B_2^d$
    & $0.992$ & $5.476$ & $0.0007$ & $0.294$ \\
Random $8$-dimensional subspace
    & $0.493$ & $11.028$ & $0.0032$ & $0.593$ \\
$5$-sparse unit probe
    & $0.748$ & $7.264$ & $0.0017$ & $0.391$ \\
Normalized ellipsoid
    & $0.476$ & $11.401$ & $0.0132$ & $0.613$ \\
\bottomrule
\end{tabular}
\end{table}

The maximum observed square-root shape drift is
\[
    \max_t\widehat\varepsilon_t
    =
    0.0538.
\]
For all four fixed probes, the observed normalized residual is small.
The largest value in Table~\ref{tab:shape-probe-diagnostics} is about
$1.3\times10^{-2}$ for the ellipsoid, while the ball, subspace, and
sparse probes have still smaller residuals.

The Monte Carlo residuals also lie well below the corresponding
quantities
$\widehat\chi_T\widehat\varepsilon_t$ at every recorded checkpoint.
The comparison is therefore consistent with
Theorem~\ref{thm:probe-stability}, but the bound is conservative for
these probes: its worst-case upper values are one to two orders of
magnitude larger than the observed residuals. We do not interpret this
experiment as evidence that the bound is tight, nor as a uniform
stability statement over broader or adaptive probe classes.

The role of this diagnostic is limited but useful. It shows that, along
the controlled GD trajectory, most of the variation in the tested
Fisher widths can be separated into a common trace scale and a slowly
varying fixed-probe shape factor. It does not explain why the trace
should itself be organized by loss. The population Gaussian-logistic
analysis in Section~\ref{sec:gaussian-theory} provides that additional
level-set and dynamical structure.

\section{Controlled Gaussian--Logistic Theory}
\label{sec:gaussian-theory}

We now study a population model in which the interaction between loss,
Fisher geometry, and optimization dynamics can be analysed explicitly.
The purpose is not to obtain a general loss-only description of Fisher
geometry. Instead, the model allows us to identify a distinguished
loss-parametrized branch, prove that this branch is extremal on
low-loss level sets, show that population gradient flow asymptotically
selects it, and derive the resulting late-time Fisher-width law.

Throughout this section, assume $d\geq2$ and
\[
    X\sim N(0,I_d),
    \qquad
    Y=\operatorname{sign}\langle u_\star,X\rangle,
    \qquad
    \|u_\star\|_2=1.
\]
Let
\[
    \ell(z)=\log(1+e^{-z}),
    \qquad
    \sigma(z)=\frac{1}{1+e^{-z}},
    \qquad
    v(z)=\sigma(z)(1-\sigma(z)).
\]
The population loss and model Fisher matrix are
\begin{equation}
    L(\theta)
    =
    \mathbb E\ell\!\left(Y\langle\theta,X\rangle\right),
    \qquad
    G(\theta)
    =
    \mathbb E\!\left[
        v(\langle\theta,X\rangle)XX^\top
    \right].
    \label{eq:population-loss-fisher}
\end{equation}

\subsection{Population Representation and One-Spike Fisher Geometry}
\label{sec:population-fisher-geometry}

Write
\[
    \theta=ru,
    \qquad
    r=\|\theta\|_2,
    \qquad
    \|u\|_2=1,
\]
and, relative to the teacher direction,
\begin{equation}
    \theta
    =
    \alpha u_\star+\beta w,
    \qquad
    w\perp u_\star,
    \qquad
    \|w\|_2=1.
    \label{eq:alpha-beta-decomposition}
\end{equation}
When the orthogonal component is nonzero, $w$ is chosen in its
direction. By rotational symmetry around $u_\star$, the population
gradient field remains in
$\operatorname{span}\{u_\star,w\}$, so this direction is fixed along
the population gradient-flow trajectory.

\begin{lemma}[Two-coordinate signed-margin representation]
\label{lem:two-coordinate-margin}
There exist independent random variables
\[
    Z_1\sim |N(0,1)|,
    \qquad
    Z_2\sim N(0,1),
\]
such that
\begin{equation}
    Y\langle\theta,X\rangle
    =
    \alpha Z_1+\beta Z_2.
    \label{eq:signed-margin-two-coordinate}
\end{equation}
Consequently,
\begin{equation}
    L(\alpha,\beta)
    =
    \mathbb E\ell(\alpha Z_1+\beta Z_2).
    \label{eq:loss-alpha-beta}
\end{equation}
\end{lemma}

\begin{proof}
Let
\[
    \xi=\langle u_\star,X\rangle,
    \qquad
    \zeta=\langle w,X\rangle.
\]
Then $\xi$ and $\zeta$ are independent standard Gaussians and
\[
    Y\langle\theta,X\rangle
    =
    \alpha|\xi|
    +
    \beta\,\operatorname{sign}(\xi)\zeta.
\]
Set
\[
    Z_1=|\xi|,
    \qquad
    Z_2=\operatorname{sign}(\xi)\zeta.
\]
By symmetry, $Z_2\sim N(0,1)$ and is independent of $Z_1$.
\end{proof}

Thus the loss depends separately on alignment and orthogonal
displacement. The Fisher matrix is simpler because the Gaussian
covariates are rotationally invariant. Define
\begin{equation}
    a(r)
    :=
    \mathbb E[v(rZ)Z^2],
    \qquad
    b(r)
    :=
    \mathbb E[v(rZ)],
    \qquad
    Z\sim N(0,1).
    \label{eq:a-b-definitions}
\end{equation}

\begin{proposition}[Exact one-spike Fisher geometry]
\label{prop:one-spike-fisher}
For $r>0$,
\begin{equation}
    G(ru)
    =
    b(r)(I-P_u)+a(r)P_u,
    \qquad
    P_u=uu^\top.
    \label{eq:one-spike-fisher}
\end{equation}
Hence
\begin{equation}
    \operatorname{Tr}G(ru)
    =
    F(r)
    :=
    a(r)+(d-1)b(r),
    \label{eq:fisher-trace-F}
\end{equation}
and
\begin{equation}
    G(ru)^{1/2}
    =
    \sqrt{b(r)}(I-P_u)+\sqrt{a(r)}P_u.
    \label{eq:fisher-square-root-one-spike}
\end{equation}
Both $a(r)$ and $b(r)$ are strictly decreasing for $r>0$.
\end{proposition}

\begin{proof}
By rotational invariance, take $u=e_1$ and write
\[
    X=(Z,U),
    \qquad
    Z\sim N(0,1),
    \qquad
    U\sim N(0,I_{d-1}),
\]
with $Z$ and $U$ independent. Since
$\langle ru,X\rangle=rZ$, symmetry eliminates the off-diagonal
entries. The longitudinal entry is $a(r)$ and every transverse entry
is $b(r)$, giving
\eqref{eq:one-spike-fisher}--\eqref{eq:fisher-square-root-one-spike}.

For every $z\neq0$, $v(rz)$ is strictly decreasing in $r>0$.
Integrating against the positive weights $1$ and $Z^2$ shows that both
$a$ and $b$ are strictly decreasing.
\end{proof}

Define
\[
    \kappa(r)
    :=
    \frac{a(r)}{b(r)}.
\]
Since $Z^2$ is increasing in $|Z|$, whereas $v(rZ)$ is decreasing in
$|Z|$,
\[
    a(r)
    =
    \mathbb E[Z^2v(rZ)]
    \leq
    \mathbb E[Z^2]\mathbb E[v(rZ)]
    =
    b(r).
\]
Thus $0<\kappa(r)\leq1$, and
\begin{equation}
    \overline G(ru)
    :=
    \frac{G(ru)}{\operatorname{Tr}G(ru)}
    =
    \frac{I-P_u}{d-1+\kappa(r)}
    +
    \frac{\kappa(r)}{d-1+\kappa(r)}P_u.
    \label{eq:normalized-one-spike-fisher}
\end{equation}

\begin{corollary}[Trace-normalized near-isotropy]
\label{cor:normalized-fisher-near-isotropy}
For every $\theta$,
\begin{equation}
    \left\|
        \overline G(\theta)-\frac1dI
    \right\|_{\mathrm{op}}
    \leq
    \frac1d,
    \qquad
    \left\|
        \overline G(\theta)-\frac1dI
    \right\|_F
    \leq
    \frac{1}{\sqrt{d(d-1)}}.
    \label{eq:normalized-fisher-isotropy-bounds}
\end{equation}
\end{corollary}

\begin{proof}
For $\theta\neq0$, the longitudinal eigenvalue of
$\overline G-I/d$ equals
\[
    -\frac{(d-1)(1-\kappa(r))}
    {d(d-1+\kappa(r))},
\]
while the transverse eigenvalue equals
\[
    \frac{1-\kappa(r)}
    {d(d-1+\kappa(r))}
\]
with multiplicity $d-1$. The bounds follow from
$0<\kappa(r)\leq1$. At $\theta=0$,
$G(0)=\frac14I$, so $\overline G(0)=I/d$ and the bounds hold
trivially.
\end{proof}

The bounds concern only the trace-normalized Fisher matrix. The
unnormalized longitudinal and transverse eigenvalues can have
different asymptotic scales.

\subsection{The Aligned Branch as a Global Envelope}
\label{sec:aligned-envelope}

On the teacher-aligned ray, define
\begin{equation}
    L_\star(r)
    :=
    L(ru_\star)
    =
    \mathbb E\ell(r|Z|).
    \label{eq:aligned-loss}
\end{equation}
Differentiation under the expectation gives
\begin{equation}
    L_\star'(r)
    =
    -\mathbb E\!\left[
        |Z|\sigma(-r|Z|)
    \right]
    <0.
    \label{eq:aligned-loss-derivative}
\end{equation}
Moreover,
\[
    L_\star(0)=\log2,
    \qquad
    L_\star(r)\to0
    \quad\text{as }r\to\infty.
\]
Hence $L_\star$ is a bijection from $[0,\infty)$ onto
$(0,\log2]$.

\begin{proposition}[Exact aligned loss parametrization]
\label{prop:exact-aligned-eos}
Let
\[
    r_\star=L_\star^{-1}.
\]
For a fixed compact probe $T$, define
\begin{equation}
    \Psi_T(\ell)
    :=
    w_F\!\left(
        T;r_\star(\ell)u_\star
    \right).
    \label{eq:aligned-eos-function}
\end{equation}
Then
\begin{equation}
    w_F(T;ru_\star)
    =
    \Psi_T(L_\star(r))
    \qquad
    \text{for all }r\geq0.
    \label{eq:exact-aligned-eos}
\end{equation}
\end{proposition}

\begin{proof}
This follows directly from
\[
    r_\star(L_\star(r))=r.
\]
\end{proof}

The proposition only reparametrizes the aligned ray. The central
geometric result is that this ray is extremal within every low-loss
level set.

\begin{theorem}[Global aligned-envelope principle]
\label{thm:global-aligned-envelope}
Let
\[
    \theta
    =
    \alpha u_\star+\beta w,
    \qquad
    r=\sqrt{\alpha^2+\beta^2},
\]
and suppose
\[
    L(\theta)<\log2.
\]
Then $\alpha>0$ and
\begin{equation}
    L(\alpha,\beta)-L_\star(r)
    \geq
    \frac12 b(r)\beta^2.
    \label{eq:global-misalignment-penalty}
\end{equation}
Equality holds if and only if $\beta=0$.

Let
\begin{equation}
    r_L
    :=
    L_\star^{-1}(L(\theta)).
    \label{eq:matched-loss-radius}
\end{equation}
Then
\begin{equation}
    r\geq r_L,
    \label{eq:minimum-radius-fixed-loss}
\end{equation}
with strict inequality when $\beta\neq0$. Consequently,
\begin{equation}
    \operatorname{Tr}G(\theta)
    \leq
    F(r_L),
    \label{eq:trace-envelope}
\end{equation}
and
\begin{equation}
    w_F(B_2^d;\theta)
    \leq
    \Psi_{B_2^d}(L(\theta)).
    \label{eq:ball-width-envelope}
\end{equation}
Both inequalities are strict away from the aligned branch.
\end{theorem}

\begin{proof}
First,
\[
    \mathbb E[\alpha Z_1+\beta Z_2]
    =
    \alpha\sqrt{\frac2\pi}.
\]
Since $\ell$ is convex, Jensen's inequality gives
\[
    L(\alpha,\beta)
    \geq
    \ell\!\left(
        \alpha\sqrt{\frac2\pi}
    \right).
\]
If $\alpha\leq0$, the right-hand side is at least $\log2$.
Hence $L(\theta)<\log2$ implies $\alpha>0$.

Because $L(\alpha,\beta)$ is even in $\beta$, assume
$\beta\geq0$. Fix $r$ and define, for $0\leq s\leq r$,
\[
    H_r(s)
    :=
    L\!\left(
        s,\sqrt{r^2-s^2}
    \right).
\]
From \eqref{eq:loss-alpha-beta},
\[
    \partial_\alpha L
    =
    -A(\alpha,\beta),
    \qquad
    A(\alpha,\beta)
    :=
    \mathbb E[
        Z_1\sigma(-\alpha Z_1-\beta Z_2)
    ]
    >0.
\]
Stein's identity in $Z_2$ gives
\begin{equation}
    \partial_\beta L(\alpha,\beta)
    =
    \beta\,b\!\left(
        \sqrt{\alpha^2+\beta^2}
    \right).
    \label{eq:beta-loss-derivative}
\end{equation}
Along the fixed-radius curve defining $H_r$,
\[
\begin{aligned}
    H_r'(s)
    &=
    -A\!\left(
        s,\sqrt{r^2-s^2}
    \right)
    -
    s\,b(r)
    <0.
\end{aligned}
\]
Therefore
\[
\begin{aligned}
    L(\alpha,\beta)-L_\star(r)
    &=
    H_r(\alpha)-H_r(r) \\
    &=
    \int_\alpha^r
    \left[
        A\!\left(
            s,\sqrt{r^2-s^2}
        \right)
        +
        s\,b(r)
    \right]ds \\
    &\geq
    \frac{b(r)}2(r^2-\alpha^2) \\
    &=
    \frac12b(r)\beta^2.
\end{aligned}
\]
If $\beta\neq0$, then $\alpha<r$ and the integral involving $A$
is strictly positive, so the inequality is strict.

By definition,
\[
    L_\star(r_L)
    =
    L(\theta)
    \geq
    L_\star(r).
\]
Since $L_\star$ is strictly decreasing,
$r_L\leq r$, with strict inequality away from alignment.

Proposition~\ref{prop:one-spike-fisher} shows that
$F(r)$ is strictly decreasing, which gives
\eqref{eq:trace-envelope}. For the Euclidean ball, define
\begin{equation}
    W(r)
    :=
    w_F(B_2^d;ru)
    =
    \mathbb E
    \sqrt{
        a(r)g_1^2
        +
        b(r)\sum_{j=2}^d g_j^2
    }.
    \label{eq:ball-width-radial}
\end{equation}
This quantity depends only on $r$ and is strictly decreasing because
both $a(r)$ and $b(r)$ are strictly decreasing. Hence
\[
    W(r)
    \leq
    W(r_L)
    =
    \Psi_{B_2^d}(L(\theta)),
\]
with strict inequality when $\beta\neq0$.
\end{proof}

Thus, at a fixed loss below $\log2$, the aligned state has the
smallest parameter norm and maximizes both Fisher trace and
Euclidean-ball Fisher width. In particular, the aligned
loss-parametrized relation is an envelope rather than a global
identity.

\begin{corollary}[No global loss-only Fisher width]
\label{cor:no-global-eos}
For every $\ell_0\in(0,\log2)$, there exist aligned and non-aligned
parameters with loss $\ell_0$ and different Euclidean-ball Fisher
widths. Consequently, no function $\Psi$ can satisfy
\[
    w_F(B_2^d;\theta)=\Psi(L(\theta))
\]
throughout parameter space.
\end{corollary}

\begin{proof}
Let
\[
    r_0=L_\star^{-1}(\ell_0)
\]
and choose $\alpha_1>r_0$. Then
$L(\alpha_1,0)<\ell_0$. For fixed $\alpha_1$,
\eqref{eq:beta-loss-derivative} shows that
$\beta\mapsto L(\alpha_1,\beta)$ is strictly increasing for
$\beta>0$.

Moreover, $\ell(z)\geq(-z)_+$, and therefore
\[
    \frac{L(\alpha_1,\beta)}{\beta}
    \geq
    \mathbb E
    \left[
        \left(
            -Z_2-\frac{\alpha_1}{\beta}Z_1
        \right)_+
    \right]
    \longrightarrow
    \mathbb E[(-Z_2)_+]
    >0
\]
as $\beta\to\infty$. Hence
$L(\alpha_1,\beta)\to\infty$, so by continuity there exists
$\beta_1>0$ such that
\[
    L(\alpha_1,\beta_1)=\ell_0.
\]
The strict part of
Theorem~\ref{thm:global-aligned-envelope} then gives
\[
    w_F(B_2^d;\alpha_1u_\star+\beta_1w)
    <
    w_F(B_2^d;r_0u_\star).
\]
\end{proof}

\subsection{Dynamical Selection of the Aligned Branch}
\label{sec:dynamical-selection}

The envelope theorem is geometric. We next ask whether population
gradient flow actually approaches this distinguished branch.

\begin{lemma}[Exact two-coordinate gradient flow]
\label{lem:two-coordinate-flow}
Under population gradient flow
\[
    \dot\theta=-\nabla L(\theta),
\]
the coordinates in \eqref{eq:alpha-beta-decomposition} satisfy
\begin{align}
    \dot\alpha
    &=
    A(\alpha,\beta)
    :=
    \mathbb E\!\left[
        Z_1\sigma(-\alpha Z_1-\beta Z_2)
    \right]
    >0,
    \label{eq:alpha-dynamics}\\
    \dot\beta
    &=
    -\beta b(r),
    \qquad
    r=\sqrt{\alpha^2+\beta^2}.
    \label{eq:beta-dynamics}
\end{align}
Hence
\begin{equation}
    \beta(t)
    =
    \beta(0)
    \exp\left\{
        -\int_0^t b(r(s))\,ds
    \right\}.
    \label{eq:beta-explicit}
\end{equation}
\end{lemma}

\begin{proof}
Differentiating \eqref{eq:loss-alpha-beta} gives
\[
    \partial_\alpha L
    =
    -A(\alpha,\beta)
\]
and
\[
    \partial_\beta L
    =
    -\mathbb E[
        Z_2\sigma(-\alpha Z_1-\beta Z_2)
    ].
\]
Conditioning on $Z_1$ and applying Stein's identity to $Z_2$ yields
\[
    \partial_\beta L
    =
    \beta\,
    \mathbb E[
        v(\alpha Z_1+\beta Z_2)
    ].
\]
Since $v$ is even and
$\alpha Z_1+\beta Z_2$ has the same absolute-value distribution as
a centered Gaussian with variance $r^2$,
\[
    \mathbb E[
        v(\alpha Z_1+\beta Z_2)
    ]
    =
    b(r).
\]
This proves \eqref{eq:alpha-dynamics} and
\eqref{eq:beta-dynamics}; integrating the latter gives
\eqref{eq:beta-explicit}.
\end{proof}

The teacher-aligned ray is an invariant special case.

\begin{proposition}[Aligned-ray invariance]
\label{prop:aligned-ray-invariance}
If $\theta_0=r_0u_\star$, then both population gradient flow
\[
    \dot\theta=-\nabla L(\theta)
\]
and damped natural-gradient flow
\[
    \dot\theta
    =
    -\bigl(G(\theta)+\lambda I\bigr)^{-1}
    \nabla L(\theta),
    \qquad
    \lambda\geq0,
\]
remain on the aligned ray. Their radial equations are
\[
    \dot r=-L_\star'(r)>0
\]
and
\[
    \dot r
    =
    -\frac{L_\star'(r)}{a(r)+\lambda}
    >0,
\]
respectively.
\end{proposition}

\begin{proof}
On the aligned ray,
\[
    \nabla L(ru_\star)
    =
    L_\star'(r)u_\star,
    \qquad
    G(ru_\star)u_\star
    =
    a(r)u_\star.
\]
Thus both vector fields are parallel to $u_\star$, and the stated
radial equations follow.
\end{proof}

For the asymptotic dynamics we use the following scalar estimates.

\begin{lemma}[Large-radius drift asymptotics]
\label{lem:large-radius-drift}
As $r\to\infty$,
\begin{equation}
    b(r)
    \sim
    \frac{1}{\sqrt{2\pi}}\frac1r.
    \label{eq:b-drift-asymptotic}
\end{equation}
Moreover, along any sequence with
$\alpha\to\infty$ and $\beta\to0$,
\begin{equation}
    A(\alpha,\beta)
    \sim
    \frac{\pi^2}{6\sqrt{2\pi}}\frac1{\alpha^2}.
    \label{eq:A-drift-asymptotic}
\end{equation}
\end{lemma}

\begin{proof}
Let
\[
    \phi(z)
    =
    \frac{1}{\sqrt{2\pi}}e^{-z^2/2}.
\]
By symmetry and the change of variables $s=rz$,
\[
    b(r)
    =
    \frac2r
    \int_0^\infty
    v(s)\phi(s/r)\,ds.
\]
Since
\[
    \int_0^\infty v(s)\,ds=\frac12,
\]
dominated convergence gives \eqref{eq:b-drift-asymptotic}.

For the second claim, let
\[
    h(z)
    =
    \sqrt{\frac2\pi}e^{-z^2/2},
    \qquad z\geq0,
\]
be the density of $Z_1$. Then
\[
\begin{aligned}
    A(\alpha,\beta)
    &=
    \frac1{\alpha^2}
    \int_0^\infty
    s\,h(s/\alpha)
    \mathbb E_{Z_2}
    \left[
        \sigma(-s-\beta Z_2)
    \right]ds.
\end{aligned}
\]
As $\alpha\to\infty$ and $\beta\to0$, the integrand after multiplication
by $\alpha^2$ converges to
\[
    s\sqrt{\frac2\pi}\,\frac1{1+e^s}.
\]
For sufficiently small $|\beta|$,
\[
    \mathbb E[
        \sigma(-s-\beta Z_2)
    ]
    \leq
    e^{-s+\beta^2/2},
\]
which provides an integrable dominating function. Hence
\[
\begin{aligned}
    \alpha^2 A(\alpha,\beta)
    &\longrightarrow
    \sqrt{\frac2\pi}
    \int_0^\infty
    \frac{s}{1+e^s}\,ds \\
    &=
    \sqrt{\frac2\pi}\,
    \frac{\pi^2}{12}
    =
    \frac{\pi^2}{6\sqrt{2\pi}}.
\end{aligned}
\]
\end{proof}

\begin{theorem}[Asymptotic dynamical selection]
\label{thm:asymptotic-dynamical-selection}
Let $\theta(t)$ be any population gradient-flow trajectory. Then
\begin{equation}
    \alpha(t)\longrightarrow\infty,
    \qquad
    \beta(t)\longrightarrow0.
    \label{eq:asymptotic-alignment}
\end{equation}
Consequently,
\begin{equation}
    \frac{\alpha(t)}{\|\theta(t)\|_2}
    \longrightarrow1.
    \label{eq:rho-to-one}
\end{equation}
More precisely,
\begin{equation}
    \frac{\alpha(t)^3}{t}
    \longrightarrow
    \frac{\pi^2}{2\sqrt{2\pi}}.
    \label{eq:alpha-cubic-rate}
\end{equation}
If $\beta(0)\neq0$, then
\begin{equation}
    \frac{\log|\beta(t)|}{\alpha(t)^2}
    \longrightarrow
    -\frac{3}{\pi^2}.
    \label{eq:beta-alpha-rate}
\end{equation}
In particular,
\[
    \alpha(t)\asymp t^{1/3},
\]
while $|\beta(t)|$ decays on a stretched-exponential $t^{2/3}$
scale when $\beta(0)\neq0$.
\end{theorem}

\begin{proof}
Equation \eqref{eq:beta-dynamics} shows that
$|\beta(t)|$ is nonincreasing and therefore bounded.

Suppose that $\alpha(t)$ were bounded above. Since $\dot\alpha>0$,
the pair $(\alpha(t),\beta(t))$ would remain in a compact subset of
$\mathbb R^2$. The continuous function $A(\alpha,\beta)$ is strictly
positive at every finite point, so it would have a positive lower
bound on this compact set. Then $\dot\alpha\geq c>0$, contradicting
boundedness. Hence
\[
    \alpha(t)\to\infty.
\]

Also,
\[
    \dot\alpha
    \leq
    \mathbb E Z_1
    =
    \sqrt{\frac2\pi},
\]
so, since $|\beta(t)|$ is bounded,
\[
    r(t)\leq C(1+t)
\]
for some constant $C$. By \eqref{eq:b-drift-asymptotic}, there exists
$c>0$ such that
\[
    b(r)\geq\frac{c}{r}
\]
for all sufficiently large $r$. Since $\alpha(t)\to\infty$, also
$r(t)\to\infty$, and therefore
\[
    \int_0^\infty b(r(s))\,ds
    =
    \infty.
\]
Equation \eqref{eq:beta-explicit} now gives
\[
    \beta(t)\to0.
\]
Together with $\alpha(t)\to\infty$, this yields
\[
    \frac{r(t)}{\alpha(t)}\to1,
\]
and hence \eqref{eq:rho-to-one}.

Lemma~\ref{lem:large-radius-drift} now applies along the trajectory:
\[
    \dot\alpha
    \sim
    \frac{c_A}{\alpha^2},
    \qquad
    c_A
    :=
    \frac{\pi^2}{6\sqrt{2\pi}}.
\]
Thus
\[
    \frac{d}{dt}\alpha^3
    =
    3\alpha^2\dot\alpha
    \longrightarrow
    3c_A,
\]
which gives
\[
    \frac{\alpha(t)^3}{t}
    \longrightarrow
    3c_A
    =
    \frac{\pi^2}{2\sqrt{2\pi}}.
\]

Assume now that $\beta(0)\neq0$. Its sign is preserved, so
$\log|\beta(t)|$ is well defined, and
\[
    \frac{d}{dt}\log|\beta(t)|
    =
    -b(r(t)).
\]
Both $-\log|\beta(t)|$ and $\alpha(t)^2$ diverge. Moreover,
using $r(t)/\alpha(t)\to1$ together with
\eqref{eq:b-drift-asymptotic} and
\eqref{eq:A-drift-asymptotic},
\[
\begin{aligned}
    \frac{
        \frac{d}{dt}\log|\beta(t)|
    }{
        \frac{d}{dt}\alpha(t)^2
    }
    &=
    -\frac{b(r(t))}
    {2\alpha(t)A(\alpha(t),\beta(t))} \\
    &\longrightarrow
    -\frac{
        1/\sqrt{2\pi}
    }{
        2\pi^2/(6\sqrt{2\pi})
    }
    =
    -\frac3{\pi^2}.
\end{aligned}
\]
L'Hospital's rule yields \eqref{eq:beta-alpha-rate}.
\end{proof}

The theorem supplies the dynamical link missing from the level-set
geometry: arbitrary population gradient-flow trajectories
asymptotically approach the aligned branch. We next quantify finite
deviations from that branch and determine its low-loss Fisher-width
scale.

\subsection{Near-Branch Stability and the Asymptotic Fisher-Width Law}
\label{sec:stability-asymptotics}

Let
\[
    \theta=\alpha u_\star+\beta w,
    \qquad
    \alpha>0,
\]
and define the aligned parameter with the same loss by
\begin{equation}
    \widehat\alpha
    :=
    L_\star^{-1}(L(\alpha,\beta)).
    \label{eq:matched-aligned-coordinate}
\end{equation}
Then
\[
    \Psi_T(L(\alpha,\beta))
    =
    w_F(T;\widehat\alpha u_\star).
\]

Because $L(\alpha,\beta)$ is even in $\beta$,
\[
    \partial_\beta L(\alpha,0)=0,
    \qquad
    \partial_\beta^2L(\alpha,0)=b(\alpha).
\]
Uniformly for $\alpha$ in compact subsets of $(0,\infty)$,
\begin{equation}
    L(\alpha,\beta)
    =
    L_\star(\alpha)
    +
    \frac12b(\alpha)\beta^2
    +
    O(\beta^4).
    \label{eq:loss-beta-expansion}
\end{equation}
Writing
\[
    g(\alpha)
    :=
    -L_\star'(\alpha)>0,
\]
local inversion gives
\begin{equation}
    \widehat\alpha
    =
    \alpha
    -
    \frac{b(\alpha)}{2g(\alpha)}\beta^2
    +
    O(\beta^4).
    \label{eq:matched-alpha-expansion}
\end{equation}

\begin{theorem}[Near-aligned Fisher-width stability]
\label{thm:near-aligned-stability}
Let
\[
    J=[\alpha_-,\alpha_+]\subset(0,\infty)
\]
be compact. For sufficiently small $|\beta|$, uniformly over
$\alpha\in J$:

\begin{enumerate}
    \item For every fixed nonempty compact probe $T$,
    \begin{equation}
        \left|
            w_F(T;\alpha u_\star+\beta w)
            -
            \Psi_T(L(\alpha,\beta))
        \right|
        \leq
        C_{T,J}|\beta|.
        \label{eq:general-probe-near-aligned}
    \end{equation}

    \item If $T$ is orthogonally invariant, then
    \begin{equation}
        \left|
            w_F(T;\alpha u_\star+\beta w)
            -
            \Psi_T(L(\alpha,\beta))
        \right|
        \leq
        C_{T,J}^{\mathrm{rot}}\beta^2.
        \label{eq:rotation-invariant-quadratic}
    \end{equation}

    \item For $T=B_2^d$, let $W$ be defined by
    \eqref{eq:ball-width-radial}. Then
    \begin{equation}
    \begin{aligned}
        &w_F(B_2^d;\alpha u_\star+\beta w)
        -
        \Psi_{B_2^d}(L(\alpha,\beta))
        \\
        &\qquad=
        \frac{W'(\alpha)}2
        \left(
            \frac1\alpha
            +
            \frac{b(\alpha)}{g(\alpha)}
        \right)\beta^2
        +
        O_J(\beta^4).
    \end{aligned}
    \label{eq:ball-width-quadratic-expansion}
    \end{equation}
\end{enumerate}
\end{theorem}

\begin{proof}
Let
\[
    r=\sqrt{\alpha^2+\beta^2},
    \qquad
    u=\frac{\alpha u_\star+\beta w}{r}.
\]
Since $\alpha$ remains in a compact subset of $(0,\infty)$,
\[
    r-\alpha
    =
    \frac{\beta^2}{r+\alpha}
    =
    O_J(\beta^2).
\]
Together with \eqref{eq:matched-alpha-expansion},
\[
    |r-\widehat\alpha|
    =
    O_J(\beta^2).
\]

From Proposition~\ref{prop:one-spike-fisher},
\[
    G(ru)^{1/2}
    =
    \sqrt{b(r)}I
    +
    \bigl(\sqrt{a(r)}-\sqrt{b(r)}\bigr)P_u.
\]
On the relevant compact radial interval, the scalar coefficients are
Lipschitz, while
\[
    \|P_u-P_{u_\star}\|_{\mathrm{op}}
    =
    \sin\angle(u,u_\star)
    =
    \frac{|\beta|}{r}.
\]
Hence
\[
    \left\|
        G(ru)^{1/2}
        -
        G(\widehat\alpha u_\star)^{1/2}
    \right\|_{\mathrm{op}}
    =
    O_J(|\beta|).
\]
The Gaussian-width perturbation bound in
Theorem~\ref{thm:probe-stability} proves
\eqref{eq:general-probe-near-aligned}.

If $T$ is orthogonally invariant, its Fisher width depends only on the
radius $r$, so the first-order rotation term disappears. Since
$|r-\widehat\alpha|=O_J(\beta^2)$,
\eqref{eq:rotation-invariant-quadratic} follows.

For the Euclidean ball,
\[
    r
    =
    \alpha+\frac{\beta^2}{2\alpha}
    +
    O_J(\beta^4),
\]
and \eqref{eq:matched-alpha-expansion} gives
\[
    r-\widehat\alpha
    =
    \frac12
    \left(
        \frac1\alpha
        +
        \frac{b(\alpha)}{g(\alpha)}
    \right)\beta^2
    +
    O_J(\beta^4).
\]
Taylor expansion of $W$ at $\alpha$ yields
\eqref{eq:ball-width-quadratic-expansion}.
\end{proof}

For a general fixed probe, rotation of the Fisher eigendirection can
contribute at order $|\beta|$. The theorem therefore provides an
$O(|\beta|)$ upper bound but does not assert that this order is sharp.
For rotation-invariant probes, the orientation dependence cancels and
the residual is quadratic.

We now determine the low-loss scale on the aligned branch.

\begin{proposition}[Aligned low-loss asymptotics]
\label{prop:aligned-late-asymptotics}
As $r\to\infty$,
\begin{align}
    L_\star(r)
    &\sim
    \frac{\pi^2}{6\sqrt{2\pi}}\frac1r,
    \label{eq:loss-asymptotic}\\
    b(r)
    &\sim
    \frac{1}{\sqrt{2\pi}}\frac1r,
    \label{eq:b-asymptotic}\\
    a(r)
    &\sim
    \frac{\pi^2}{3\sqrt{2\pi}}\frac1{r^3}.
    \label{eq:a-asymptotic}
\end{align}
Consequently,
\begin{equation}
    \kappa(r)
    \sim
    \frac{\pi^2}{3r^2},
    \qquad
    \frac{b(r)}{L_\star(r)}
    \longrightarrow
    \frac6{\pi^2}.
    \label{eq:kappa-late}
\end{equation}
For $d\geq2$,
\begin{equation}
    w_F(B_2^d;ru_\star)
    \sim
    (2\pi)^{-1/4}
    \mathbb E[\chi_{d-1}]
    r^{-1/2},
    \label{eq:ball-width-r-asymptotic}
\end{equation}
where $\chi_{d-1}$ denotes a chi random variable with $d-1$ degrees of
freedom. Therefore
\begin{equation}
    \frac{
        w_F(B_2^d;ru_\star)
    }{
        \sqrt{L_\star(r)}
    }
    \longrightarrow
    \frac{\sqrt6}{\pi}
    \mathbb E[\chi_{d-1}].
    \label{eq:ball-width-loss-asymptotic}
\end{equation}
\end{proposition}

\begin{proof}
The asymptotic for $b(r)$ was obtained in
Lemma~\ref{lem:large-radius-drift}. For the loss,
\[
    L_\star(r)
    =
    \frac2r
    \int_0^\infty
    \log(1+e^{-s})\phi(s/r)\,ds.
\]
Using
\[
    \int_0^\infty
    \log(1+e^{-s})\,ds
    =
    \frac{\pi^2}{12},
\]
dominated convergence gives \eqref{eq:loss-asymptotic}.

Similarly,
\[
    a(r)
    =
    \frac2{r^3}
    \int_0^\infty
    s^2v(s)\phi(s/r)\,ds,
\]
and
\[
    \int_0^\infty s^2v(s)\,ds
    =
    \frac{\pi^2}{6},
\]
which yields \eqref{eq:a-asymptotic}. The statements in
\eqref{eq:kappa-late} follow by taking ratios.

For the ball,
\[
    W(r)
    =
    \mathbb E
    \sqrt{
        a(r)g_1^2
        +
        b(r)\sum_{j=2}^d g_j^2
    }.
\]
Since $a(r)/b(r)\to0$,
\[
    \frac{W(r)}{\sqrt{b(r)}}
    \longrightarrow
    \mathbb E[\chi_{d-1}],
\]
by dominated convergence. Combining this with
\eqref{eq:b-asymptotic} gives
\eqref{eq:ball-width-r-asymptotic}, and division by
\eqref{eq:loss-asymptotic} gives
\eqref{eq:ball-width-loss-asymptotic}.
\end{proof}

The preceding proposition concerns the exactly aligned ray.
The dynamical selection theorem allows the same leading law to be
transferred to an arbitrary population gradient-flow trajectory.

\begin{corollary}[Dynamic Fisher-width equation of state]
\label{cor:dynamic-fisher-eos}
Let $\theta(t)$ follow population gradient flow. Then, for $d\geq2$,
\begin{equation}
    \frac{
        w_F(B_2^d;\theta(t))
    }{
        \sqrt{L(\theta(t))}
    }
    \longrightarrow
    \frac{\sqrt6}{\pi}
    \mathbb E[\chi_{d-1}].
    \label{eq:dynamic-fisher-eos}
\end{equation}
Moreover,
\begin{equation}
    L(\theta(t))
    \asymp
    t^{-1/3},
    \qquad
    w_F(B_2^d;\theta(t))
    \asymp
    t^{-1/6}.
    \label{eq:dynamic-loss-width-rates}
\end{equation}
\end{corollary}

\begin{proof}
By Theorem~\ref{thm:asymptotic-dynamical-selection},
\[
    \alpha(t)\to\infty,
    \qquad
    \beta(t)\to0,
    \qquad
    \frac{r(t)}{\alpha(t)}\to1.
\]

If $\beta(0)=0$, then
\eqref{eq:beta-dynamics} gives $\beta(t)\equiv0$, and hence
\[
    L(\theta(t))=L_\star(\alpha(t)).
\]
If $\beta(0)\neq0$, then
\eqref{eq:beta-alpha-rate} implies
\[
    |\beta(t)|
    =
    \exp\!\left\{
        -\left(\frac{3}{\pi^2}+o(1)\right)\alpha(t)^2
    \right\}
    =
    o(\alpha(t)^{-1}).
\]
Since the logistic loss is $1$-Lipschitz in its margin,
\[
    |L(\alpha,\beta)-L_\star(\alpha)|
    \leq
    |\beta|\,\mathbb E|Z_2|.
\]
Thus, in both cases,
\[
    L(\theta(t))
    \sim
    L_\star(\alpha(t))
    \sim
    \frac{\pi^2}{6\sqrt{2\pi}}
    \frac1{\alpha(t)}.
\]

The Euclidean-ball Fisher width depends only on
$r=\|\theta\|_2$. Since $r(t)/\alpha(t)\to1$,
Proposition~\ref{prop:aligned-late-asymptotics} gives
\[
    w_F(B_2^d;\theta(t))
    \sim
    (2\pi)^{-1/4}
    \mathbb E[\chi_{d-1}]
    \alpha(t)^{-1/2}.
\]
Taking the ratio proves \eqref{eq:dynamic-fisher-eos}. Finally,
\eqref{eq:alpha-cubic-rate} gives
$\alpha(t)\asymp t^{1/3}$, which yields
\eqref{eq:dynamic-loss-width-rates}.
\end{proof}

The controlled model therefore gives three distinct levels of
structure. First, the aligned branch is an extremal envelope on each
low-loss level set. Second, population gradient flow asymptotically
selects that branch. Third, the Euclidean-ball Fisher width along the
actual trajectory satisfies the explicit late-time relation
\eqref{eq:dynamic-fisher-eos}. These conclusions are specific to the
population Gaussian--logistic model and do not imply a global
loss-only law outside this setting. Section~\ref{sec:finite-sample}
next studies finite-sample evaluation and numerical validation of the
corresponding population predictions.

\section{Finite-Sample Control and Controlled Validation}
\label{sec:finite-sample}

The results of Section~\ref{sec:gaussian-theory} are population
statements. This section separates three different questions. First,
we quantify how the empirical Fisher trace behaves when a fixed
parameter is evaluated on an independent Gaussian sample. Second, we
numerically test the population predictions for the global aligned
envelope and dynamical selection. Third, we distinguish these
population diagnostics from quantities computed on held-out
finite-sample checkpoints.

The theorem-facing numerical experiments use population quantities
computed from the two-coordinate representation and Gaussian
quadrature. They therefore test predictions derived in
Section~\ref{sec:gaussian-theory} without introducing sampling noise.
Held-out empirical quantities are treated separately and are not used
to estimate the population asymptotic constants.

\subsection{Finite-Sample Trace Control}
\label{sec:finite-sample-trace}

For fixed vectors $x_1,\ldots,x_n\in\mathbb R^d$, define
\begin{equation}
    \widehat G_n(\theta)
    :=
    \frac1n
    \sum_{i=1}^n
    v(\langle\theta,x_i\rangle)x_ix_i^\top
    \label{eq:empirical-fisher}
\end{equation}
and
\begin{equation}
    \widehat m_n(\theta)
    :=
    \frac1n
    \sum_{i=1}^n
    v(\langle\theta,x_i\rangle).
    \label{eq:empirical-mean-variance}
\end{equation}
Taking the trace gives the deterministic identity
\begin{equation}
    \operatorname{Tr}\widehat G_n(\theta)
    =
    \frac1n
    \sum_{i=1}^n
    v(\langle\theta,x_i\rangle)\|x_i\|_2^2,
    \label{eq:exact-empirical-trace}
\end{equation}
and therefore
\begin{equation}
    \operatorname{Tr}\widehat G_n(\theta)
    -
    d\,\widehat m_n(\theta)
    =
    \frac1n
    \sum_{i=1}^n
    v(\langle\theta,x_i\rangle)
    \bigl(\|x_i\|_2^2-d\bigr).
    \label{eq:trace-variance-residual}
\end{equation}

Assume now
\[
    X_1,\ldots,X_n
    \stackrel{\mathrm{iid}}{\sim}
    N(0,I_d),
\]
and let $\theta$ be fixed independently of this sample. Write
\[
    r=\|\theta\|_2,
    \qquad
    m(r):=b(r),
    \qquad
    \kappa(r):=\frac{a(r)}{b(r)}.
\]

\begin{proposition}[Pointwise empirical trace control]
\label{prop:pointwise-trace-concentration}
There exist universal constants $C,c>0$ such that
\begin{equation}
    \mathbb E\!\left[
        \operatorname{Tr}\widehat G_n(\theta)
        -
        d\,\widehat m_n(\theta)
    \right]
    =
    m(r)(\kappa(r)-1).
    \label{eq:trace-bias}
\end{equation}
Moreover, for every $\delta\in(0,1)$, with probability at least
\[
    1-\delta-2e^{-cnm(r)},
\]
\begin{align}
&
\left|
    \operatorname{Tr}\widehat G_n(\theta)
    -
    d\,\widehat m_n(\theta)
    -
    m(r)(\kappa(r)-1)
\right|
\nonumber\\
&\qquad\leq
C\left[
    \sqrt{
        \frac{
            d\,m(r)\log(4/\delta)
        }{n}
    }
    +
    \frac{\log(4/\delta)}{n}
\right].
\label{eq:pointwise-trace-concentration}
\end{align}
With probability at least
\[
    1-\delta-4e^{-cnm(r)},
\]
one also has
\begin{align}
\left|
    \frac{
        \operatorname{Tr}\widehat G_n(\theta)
    }{
        d\,\widehat m_n(\theta)
    }
    -1
\right|
\leq
C\left[
    \frac1d
    +
    \sqrt{
        \frac{
            \log(4/\delta)
        }{
            ndm(r)
        }
    }
    +
    \frac{
        \log(4/\delta)
    }{
        ndm(r)
    }
\right].
\label{eq:relative-trace-concentration}
\end{align}
\end{proposition}

The proof is given in
Appendix~\ref{app:finite-sample-details}. It uses the decomposition
\[
\begin{aligned}
&
v(rZ_i)\bigl(\|X_i\|_2^2-d\bigr)
\\
&\qquad=
v(rZ_i)(Z_i^2-1)
+
v(rZ_i)
\bigl(\|U_i\|_2^2-(d-1)\bigr),
\end{aligned}
\]
followed by scalar and conditional Bernstein bounds.

The deterministic bias in \eqref{eq:trace-bias} reflects the one-spike
population Fisher geometry:
\[
    \frac{
        |m(r)(\kappa(r)-1)|
    }{
        d\,m(r)
    }
    \leq
    \frac1d.
\]
The stochastic part becomes harder to estimate when $m(r)$ is small.
The effective number of observations carrying non-negligible logistic
weight is governed by $nm(r)$, while the relative trace bound also
benefits from the factor $d$ in
\eqref{eq:relative-trace-concentration}.

\subsection{Independent Evaluation at Finite Checkpoints}
\label{sec:independent-evaluation}

Proposition~\ref{prop:pointwise-trace-concentration} requires the
evaluated parameter to be independent of the data used to estimate its
Fisher matrix. To apply the result to a trained trajectory, we separate
training and evaluation samples,
\[
    \mathcal D_{\mathrm{tr}},
    \qquad
    \mathcal D_{\mathrm{ev}}.
\]
The optimizer uses only $\mathcal D_{\mathrm{tr}}$ and records
\[
    \Theta_{\mathrm{tr}}
    =
    \{\theta_{t_1},\ldots,\theta_{t_N}\}.
\]
At each checkpoint we compute
\begin{align}
    \widehat L_{\mathrm{ev}}(\theta)
    &:=
    \frac1{n_{\mathrm{ev}}}
    \sum_{i=1}^{n_{\mathrm{ev}}}
    \ell\!\left(
        Y_i^{\mathrm{ev}}
        \langle\theta,X_i^{\mathrm{ev}}\rangle
    \right),
    \label{eq:evaluation-loss}\\
    \widehat G_{\mathrm{ev}}(\theta)
    &:=
    \frac1{n_{\mathrm{ev}}}
    \sum_{i=1}^{n_{\mathrm{ev}}}
    v(\langle\theta,X_i^{\mathrm{ev}}\rangle)
    X_i^{\mathrm{ev}}X_i^{\mathrm{ev}\top},
    \label{eq:evaluation-fisher}\\
    \widehat m_{\mathrm{ev}}(\theta)
    &:=
    \frac1{n_{\mathrm{ev}}}
    \sum_{i=1}^{n_{\mathrm{ev}}}
    v(\langle\theta,X_i^{\mathrm{ev}}\rangle).
    \label{eq:evaluation-variance}
\end{align}

Conditional on $\mathcal D_{\mathrm{tr}}$, the recorded checkpoints
are fixed. The same evaluation sample may be reused across
checkpoints; the resulting estimates are dependent, but independence
between checkpoints is not required for a union bound.

\begin{corollary}[Finite-checkpoint control]
\label{cor:finite-checkpoint-control}
Let
\[
    m_j=m(\|\theta_{t_j}\|_2),
    \qquad
    m_{\min}:=\min_{1\leq j\leq N}m_j>0,
\]
and set
\[
    u_N=\log\left(\frac{4N}{\delta}\right).
\]
Conditional on $\mathcal D_{\mathrm{tr}}$, with probability at least
\[
    1-\delta-2N e^{-c n_{\mathrm{ev}}m_{\min}},
\]
the bound
\begin{align}
&
\left|
    \operatorname{Tr}
    \widehat G_{\mathrm{ev}}(\theta_{t_j})
    -
    d\,\widehat m_{\mathrm{ev}}(\theta_{t_j})
    -
    m_j(\kappa_j-1)
\right|
\nonumber\\
&\qquad\leq
C\left[
    \sqrt{
        \frac{
            d\,m_j u_N
        }{
            n_{\mathrm{ev}}
        }
    }
    +
    \frac{u_N}{n_{\mathrm{ev}}}
\right]
\label{eq:finite-checkpoint-bound}
\end{align}
holds simultaneously for $j=1,\ldots,N$, where
\[
    \kappa_j=\kappa(\|\theta_{t_j}\|_2).
\]
\end{corollary}

The proof follows from
Proposition~\ref{prop:pointwise-trace-concentration} with failure
probability $\delta/N$ and a union bound. In the low-variance regime,
a useful evaluation sample must satisfy at least
\begin{equation}
    n_{\mathrm{ev}}m_{\min}
    \gtrsim
    \log(N/\delta).
    \label{eq:effective-sample-condition}
\end{equation}
Thus a fixed evaluation sample eventually loses resolution as the
population Fisher scale becomes small. The corollary is simultaneous
over the recorded finite set of checkpoints; it is not a
trajectory-uniform empirical-process result.

\subsection{Population Check of the Global Aligned Envelope}
\label{sec:envelope-validation}

We next test the pre-existing population prediction of
Theorem~\ref{thm:global-aligned-envelope}. Fix
\[
    \ell\in(0,\log2),
    \qquad
    r_\ell=L_\star^{-1}(\ell).
\]
For several values $\alpha>r_\ell$, we solve
\begin{equation}
    L(\alpha,\beta_\alpha)=\ell,
    \qquad
    \beta_\alpha>0.
    \label{eq:equal-loss-level-set}
\end{equation}
Existence and uniqueness follow from
\[
    L(\alpha,0)<\ell,
    \qquad
    \partial_\beta L(\alpha,\beta)>0
    \quad(\beta>0),
\]
together with
$L(\alpha,\beta)\to\infty$ as $\beta\to\infty$.

For each equal-loss state, let
\[
    r_\alpha
    =
    \sqrt{\alpha^2+\beta_\alpha^2}.
\]
Theorem~\ref{thm:global-aligned-envelope} predicts
\begin{equation}
    r_\alpha-r_\ell>0
    \label{eq:envelope-radius-gap}
\end{equation}
and
\begin{equation}
    W(r_\ell)-W(r_\alpha)>0,
    \qquad
    W(r):=w_F(B_2^d;ru).
    \label{eq:envelope-width-gap}
\end{equation}
We also evaluate the quantitative penalty
\begin{equation}
    R_{\mathrm{env}}(\alpha,\beta)
    :=
    \frac{
        L(\alpha,\beta)-L_\star(r)
    }{
        \tfrac12 b(r)\beta^2
    },
    \qquad
    r=\sqrt{\alpha^2+\beta^2},
    \label{eq:envelope-penalty-ratio}
\end{equation}
for which the theorem gives
\[
    R_{\mathrm{env}}(\alpha,\beta)>1
    \qquad
    (\beta\neq0).
\]

The near-aligned expansions of
Section~\ref{sec:stability-asymptotics} provide a more precise
prediction close to the aligned point. Indeed,
\[
    r
    =
    \alpha+\frac{\beta^2}{2\alpha}+O(\beta^4)
\]
and
\[
    L(\alpha,\beta)
    =
    L_\star(\alpha)
    +
    \frac12b(\alpha)\beta^2
    +
    O(\beta^4).
\]
Hence
\[
    L(\alpha,\beta)-L_\star(r)
    =
    \frac12
    \left(
        b(\alpha)+\frac{g(\alpha)}{\alpha}
    \right)\beta^2
    +
    O(\beta^4),
\]
where
\[
    g(\alpha)=-L_\star'(\alpha).
\]
Along the level set
$L(\alpha,\beta)=\ell$, one has
$\alpha\to r_\ell$ as $\beta\to0$, and therefore
\begin{equation}
    R_{\mathrm{env}}(\alpha,\beta)
    \longrightarrow
    1+
    \frac{
        g(r_\ell)
    }{
        r_\ell b(r_\ell)
    }.
    \label{eq:envelope-local-limit}
\end{equation}
Thus the lower bound in
\eqref{eq:global-misalignment-penalty} is not expected to be
asymptotically tight with ratio one.

All quantities in this experiment are population quantities evaluated
by Gaussian quadrature; no finite-sample approximation is used.

\paragraph{Numerical check.}
For the target loss
\[
    \ell=0.3,
\]
the quadrature computation gives
\[
    r_\ell
    =
    1.71399217.
\]
At every tested non-aligned point on the equal-loss level set,
\[
    r_\alpha>r_\ell,
    \qquad
    W(r_\ell)>W(r_\alpha),
    \qquad
    R_{\mathrm{env}}>1,
\]
in agreement with
Theorem~\ref{thm:global-aligned-envelope}.

The local prediction in
\eqref{eq:envelope-local-limit} is
\[
    1+
    \frac{
        g(r_\ell)
    }{
        r_\ell b(r_\ell)
    }
    =
    1.40838784.
\]
At the tested point closest to the aligned state, the numerical ratio
is
\[
    R_{\mathrm{env}}
    =
    1.40832390.
\]
The purpose of this calculation is to check the previously derived
level-set expansion; the local constant is determined by the theory,
not fitted from the numerical data.

\subsection{Population Check of Dynamical Selection and the Late-Loss Law}
\label{sec:dynamical-validation}

We now test
Theorem~\ref{thm:asymptotic-dynamical-selection} and
Corollary~\ref{cor:dynamic-fisher-eos}. Since both results concern
population gradient flow, we numerically integrate the exact
two-coordinate dynamics
\[
    \dot\alpha
    =
    A(\alpha,\beta),
    \qquad
    \dot\beta
    =
    -\beta b(r),
    \qquad
    r=\sqrt{\alpha^2+\beta^2},
\]
with $A$ and $b$ evaluated by Gaussian quadrature. For
$\beta(0)\neq0$, we equivalently evolve
\[
    q(t):=\log|\beta(t)|,
    \qquad
    \dot q(t)=-b(r(t)),
\]
when evaluating the orthogonal-decay diagnostic.

Theorem~\ref{thm:asymptotic-dynamical-selection} predicts
\[
    \alpha(t)^3
    \sim
    c_\alpha t,
    \qquad
    c_\alpha
    :=
    \frac{\pi^2}{2\sqrt{2\pi}},
\]
and, when $\beta(0)\neq0$,
\[
    -\log|\beta(t)|
    \sim
    c_\beta\alpha(t)^2,
    \qquad
    c_\beta
    :=
    \frac3{\pi^2}.
\]
To reduce the influence of the initial offsets, define
\begin{equation}
    R_\alpha(t)
    :=
    \frac{
        \alpha(t)^3-\alpha(0)^3
    }{t}
    \label{eq:alpha-rate-diagnostic}
\end{equation}
and
\begin{equation}
    R_\beta(t)
    :=
    \frac{
        -\log\!\left(
            |\beta(t)|/|\beta(0)|
        \right)
    }{
        \alpha(t)^2-\alpha(0)^2
    }.
    \label{eq:beta-rate-diagnostic}
\end{equation}
The asymptotic predictions are
\begin{equation}
    R_\alpha(t)\longrightarrow c_\alpha,
    \qquad
    R_\beta(t)\longrightarrow c_\beta.
    \label{eq:dynamical-rate-predictions}
\end{equation}

For the Fisher-width law, set
\[
    c_d
    :=
    \frac{\sqrt6}{\pi}\,
    \mathbb E[\chi_{d-1}]
    =
    \frac{\sqrt{12}}{\pi}
    \frac{
        \Gamma(d/2)
    }{
        \Gamma((d-1)/2)
    }.
\]
The corresponding population diagnostic is
\begin{equation}
    R_{\mathrm{EoS}}(t)
    :=
    \frac{
        w_F(B_2^d;\theta(t))
    }{
        c_d\sqrt{L(\theta(t))}
    },
    \label{eq:dynamic-eos-ratio}
\end{equation}
for which Corollary~\ref{cor:dynamic-fisher-eos} predicts
\[
    R_{\mathrm{EoS}}(t)\longrightarrow1.
\]

\begin{table}[htbp]
\centering
\caption{Population diagnostics for dynamical selection and the
late-loss Fisher-width law. The predictions are fixed by
Theorem~\ref{thm:asymptotic-dynamical-selection} and
Corollary~\ref{cor:dynamic-fisher-eos}; the final column reports the
numerical integration at $t=5000$.}
\label{tab:dynamical-validation}
\begin{tabular}{lcc}
\toprule
Diagnostic & Asymptotic prediction & Value at $t=5000$ \\
\midrule
$R_\alpha(t)$
&
$\displaystyle
\frac{\pi^2}{2\sqrt{2\pi}}
\approx1.9687$
&
$1.9273$
\\[2mm]
$R_\beta(t)$
&
$\displaystyle
\frac3{\pi^2}
\approx0.30396$
&
$0.31072$
\\[2mm]
$R_{\mathrm{EoS}}(t)$
&
$1$
&
$0.999587$
\\
\bottomrule
\end{tabular}
\end{table}

The three diagnostics are consistent with the corresponding
asymptotic predictions. At $t=5000$,
$R_\alpha$ and $R_\beta$ are within a few percent of their limiting
constants, while the normalized Fisher-width ratio is already very
close to one:
\[
    R_{\mathrm{EoS}}(5000)
    =
    0.999587.
\]
These finite-time values are reported only as numerical checks of the
limits in Section~\ref{sec:gaussian-theory}; no convergence rate for
$R_\alpha$, $R_\beta$, or $R_{\mathrm{EoS}}$ is inferred from the
experiment.

\subsection{Held-Out Finite-Sample Diagnostic}
\label{sec:heldout-dynamic-diagnostic}

The preceding experiments use population quantities and therefore
directly address the theoretical predictions of
Section~\ref{sec:gaussian-theory}. A finite-sample trajectory requires
a different interpretation.

For a checkpoint trained only on
$\mathcal D_{\mathrm{tr}}$, define the held-out Euclidean-ball Fisher
width from \eqref{eq:evaluation-fisher} and consider
\begin{equation}
    \widehat R_{\mathrm{EoS}}(t)
    :=
    \frac{
        \widehat w_{F,\mathrm{ev}}
        (B_2^d;\theta_t)
    }{
        c_d
        \sqrt{
            \widehat L_{\mathrm{ev}}(\theta_t)
        }
    }.
    \label{eq:heldout-dynamic-eos-ratio}
\end{equation}
This quantity is evaluated only at checkpoints for which the
effective-sample condition
\eqref{eq:effective-sample-condition} remains meaningful.

Unlike $R_{\mathrm{EoS}}(t)$, the quantity
$\widehat R_{\mathrm{EoS}}(t)$ contains finite-sample error in both the
loss and Fisher width. Proposition~\ref{prop:pointwise-trace-concentration}
and Corollary~\ref{cor:finite-checkpoint-control} control the empirical
Fisher trace, not the full Gaussian-width ratio in
\eqref{eq:heldout-dynamic-eos-ratio}. We therefore use the held-out
ratio only as a descriptive finite-sample diagnostic. It is not used
to estimate the population constant $c_d$ or to strengthen
Corollary~\ref{cor:dynamic-fisher-eos}.

The experiments in this section consequently play separate roles. The
equal-loss quadrature calculation checks the global envelope and its
near-aligned expansion. The population-flow integration checks the
dynamical-selection constants and the resulting late-loss
Fisher-width law. The independent-evaluation construction provides a
finite-sample protocol with explicit trace control, but does not turn
the population asymptotic law into a trajectory-uniform finite-sample
theorem. Section~\ref{sec:beyond-controlled} next examines which parts
of the matched-loss description remain visible outside the Gaussian
linear setting.

\section{Beyond the Controlled Regime}
\label{sec:beyond-controlled}

The theory in Sections~\ref{sec:gaussian-theory} and
\ref{sec:finite-sample} relies on linear prediction, isotropic Gaussian
covariates, and the full Fisher matrix. We now examine a nonlinear
setting in which these assumptions no longer hold. The purpose is not
to extend the Gaussian--logistic theorems to neural networks, but to
test which parts of the empirical matched-loss description remain
visible after changing the architecture and Fisher representation.

The experiment gives two main observations. First, GD and SGD
remain close at matched loss, whereas Adam follows a substantially
displaced branch. Second, all fixed probes considered below have highly
correlated temporal shapes under the same diagonal model-Fisher
geometry. Thus optimizer dependence is the clearer boundary in this
experiment; the tested fixed probes do not provide a comparable
failure mode.

\subsection{MLP Setup and Diagonal Model-Fisher Approximation}
\label{sec:mlp-setup}

We use a two-hidden-layer ReLU network with widths $256$ and $128$ for
binary MNIST classification of digits $3$ and $5$. The network has
\[
    p=233{,}985
\]
trainable parameters. The input features are standardized using the
training binary dataset. We compare full-batch GD, mini-batch SGD with
batch size $256$, and Adam, with learning rates
\[
    \eta_{\mathrm{GD}}=0.1,
    \qquad
    \eta_{\mathrm{SGD}}=0.05,
    \qquad
    \eta_{\mathrm{Adam}}=10^{-3}.
\]
Adam uses
\[
    \beta_1=0.9,
    \qquad
    \beta_2=0.999,
\]
and training runs for $600$ iterations with gradient clipping at norm
$5$. Results are aggregated over six initializations.

Let $f_\theta(x)$ denote the binary logit and
\[
    p_\theta(x)=\sigma(f_\theta(x)).
\]
For the conditional Bernoulli model, the model Fisher matrix is
\[
    G(\theta)
    =
    \mathbb E_x\!\left[
        p_\theta(x)(1-p_\theta(x))
        \nabla_\theta f_\theta(x)
        \nabla_\theta f_\theta(x)^\top
    \right].
\]
This is the model Fisher obtained after taking the conditional
expectation over the model label distribution. It should be
distinguished from an empirical Fisher constructed from gradients of
the observed labels.

At the present parameter dimension the full $p\times p$ matrix is not
formed. Instead, on a fixed subset of $500$ training examples we use
the diagonal model-Fisher approximation
\begin{equation}
    \widehat G_{\mathrm{diag}}(\theta)
    :=
    \operatorname{diag}\left(
        \frac1{500}
        \sum_{i=1}^{500}
        p_i(1-p_i)
        \nabla_\theta f_\theta(x_i)
        \odot
        \nabla_\theta f_\theta(x_i)
    \right),
    \label{eq:mlp-diagonal-model-fisher}
\end{equation}
where $p_i=p_\theta(x_i)$. The same $500$ examples are used across
all recorded checkpoints, optimizers, and initializations. This subset
is used only to evaluate the model-Fisher geometry; it is not the
independent held-out sample of Section~\ref{sec:independent-evaluation}.

For a fixed compact probe $T\subset\mathbb R^p$, define
\begin{equation}
    \widehat w_{F,\mathrm{diag}}(T;\theta)
    :=
    w\!\left(
        \widehat G_{\mathrm{diag}}(\theta)^{1/2}T
    \right).
    \label{eq:diagonal-fisher-width}
\end{equation}
For the Euclidean ball,
\begin{equation}
    \widehat w_{F,\mathrm{diag}}(B_2^p;\theta)
    =
    \mathbb E_g
    \sqrt{
        \sum_{j=1}^p
        \widehat G_{\mathrm{diag},jj}(\theta)g_j^2
    },
    \qquad
    g\sim N(0,I_p).
    \label{eq:diagonal-ball-width}
\end{equation}

The diagonal Fisher matrix and the corresponding widths are evaluated
every five training iterations. Each width is estimated with $500$
Gaussian draws. A fixed bank of Gaussian draws is reused across
checkpoints, optimizers, and initializations. The structured probes
introduced below are also generated once and then held fixed across
the complete experiment.

Equation~\eqref{eq:diagonal-fisher-width} is an actual Gaussian width
under the diagonal model-Fisher approximation. It is not the
full-Fisher width. Consequently, the results in this section concern
the geometry induced by \eqref{eq:mlp-diagonal-model-fisher} and should
not be interpreted as statements about the full Fisher matrix of the
network.

\subsection{Optimizer-Dependent Matched-Loss Behaviour}
\label{sec:mlp-optimizer-branches}

We use the matched-loss protocol of
Section~\ref{sec:optimizer-comparison}. Each optimizer is compared with
GD only over the loss interval attained by both trajectories, and no
extrapolation is used. Widths are interpolated piecewise linearly
between the recorded Fisher checkpoints, and the multiplicative
displacement is summarized by the geometric mean
\eqref{eq:matched-loss-ratio}.

For mini-batch SGD,
\begin{equation}
    \overline\rho_{\mathrm{SGD}}^{\mathrm{MLP}}
    =
    1.017\pm0.009.
    \label{eq:mlp-sgd-ratio}
\end{equation}
Thus the GD and SGD diagonal model-Fisher ball widths remain close at
matched loss, with an average displacement of less than two percent.

For Adam,
\begin{equation}
    \overline\rho_{\mathrm{Adam}}^{\mathrm{MLP}}
    =
    0.847\pm0.012.
    \label{eq:mlp-adam-shared-ratio}
\end{equation}
The displacement from the GD branch is therefore much larger than for
SGD and is consistent across the six initializations.

For comparison, the controlled logistic experiment of
Section~\ref{sec:optimizer-comparison}, which uses the full
$30\times30$ Fisher matrix, gives
\begin{equation}
    \overline\rho_{\mathrm{Adam}}^{\mathrm{logistic}}
    =
    0.995\pm0.003.
    \label{eq:logistic-adam-shared-ratio}
\end{equation}

\begin{table}[htbp]
\centering
\caption{Matched-loss comparisons with the GD reference. The logistic
quantity uses the full model-Fisher width, whereas the MLP quantity
uses the diagonal model-Fisher ball width of
\eqref{eq:diagonal-ball-width}. Values are mean $\pm$ sample standard
deviation over six initializations.}
\label{tab:optimizer-branch-comparison}
\begin{tabular}{llcl}
\toprule
System & Optimizer & $\overline\rho$ & Observation \\
\midrule
Logistic
& Adam
& $0.995\pm0.003$
& close agreement
\\
MLP
& SGD
& $1.017\pm0.009$
& close agreement
\\
MLP
& Adam
& $0.847\pm0.012$
& distinct observed branch
\\
\bottomrule
\end{tabular}
\end{table}

The comparison separates two different empirical regimes. In the
controlled logistic problem, the tested optimizers give closely
agreeing full-Fisher-width curves after conditioning on loss. In the
MLP experiment, GD and SGD still remain close, but Adam does not. Thus
loss reparametrization does not remove optimizer dependence in this
nonlinear example.

The conclusion is deliberately limited. The Adam result establishes a
reproducible branch displacement for this network, training protocol,
and diagonal model-Fisher approximation. It does not imply that Adam
must produce a different full-Fisher branch in general, nor that GD and
SGD must agree for other architectures or optimization regimes.

\subsection{Fixed-Probe Temporal Stability}
\label{sec:mlp-probe-dependence}

We next ask whether different fixed probes exhibit similar temporal
evolution under the same diagonal model-Fisher geometry. We use
\begin{align}
    T_{\mathrm{ball}}
    &:=
    B_2^p,
    \label{eq:mlp-ball-probe}\\
    T_{\mathrm{sub}}
    &:=
    \operatorname{range}(Q)\cap B_2^p,
    \qquad
    Q^\top Q=I_8,
    \label{eq:mlp-subspace-probe}\\
    T_{\mathrm{sparse}}
    &:=
    \left\{
        v\in\mathbb R^p:
        \|v\|_0\leq256,\;
        \|v\|_2\leq1
    \right\},
    \label{eq:mlp-sparse-probe}\\
    T_{\mathrm{ell}}
    &:=
    \left\{
        D_{\mathrm{ell}}^{1/2}u:
        \|u\|_2\leq1
    \right\},
    \label{eq:mlp-ellipsoid-probe}
\end{align}
where $Q$ is a fixed random orthonormal $8$-dimensional subspace and
$D_{\mathrm{ell}}$ is a fixed positive diagonal matrix normalized so
that the ellipsoid has Euclidean radius one. All probe geometry is held
fixed across checkpoints and initializations.

For the subspace probe, if
$z\sim N(0,I_p)$ is transformed by the diagonal Fisher square root,
the supremum over
$T_{\mathrm{sub}}$ is the Euclidean norm of its projection onto the
fixed subspace. For the sparse probe, the supremum is the
$\ell_2$ norm of the $256$ largest coordinates in magnitude after the
same Fisher deformation. Thus all four observables are Gaussian widths
of the stated compact probe sets rather than finite-direction or trace
proxies.

For each probe we compute
\[
    \widehat w_{F,\mathrm{diag}}(T;\theta_t)
\]
at the same Fisher checkpoints. As a descriptive measure of temporal
shape agreement with the ball probe, define
\begin{equation}
    C_T
    :=
    \operatorname{Corr}\left(
        z\!\left(
            \widehat w_{F,\mathrm{diag}}(T;\theta_t)
        \right),
        z\!\left(
            \widehat w_{F,\mathrm{diag}}(B_2^p;\theta_t)
        \right)
    \right),
    \label{eq:mlp-probe-shape-correlation}
\end{equation}
where $z(\cdot)$ denotes standardization over the recorded GD
trajectory. This statistic measures temporal shape agreement only; it
is not a matched-loss equation-of-state statistic.

We separately measure variability across initializations at matched
loss. For each probe, the six GD trajectories are interpolated on their
common observed loss interval, and the coefficient of variation is
computed across initializations and then averaged over the common loss
grid, as in Section~\ref{sec:initialization-stability}.

\begin{table}[htbp]
\centering
\caption{Fixed-probe behaviour under the diagonal model-Fisher
geometry. Shape correlation measures temporal agreement with the
Euclidean-ball trajectory and is reported as mean $\pm$ sample
standard deviation over six initializations. Matched-loss CV measures
variability across the same six initializations.}
\label{tab:mlp-probe-dependence}
\begin{tabular}{lcc}
\toprule
Probe
& Shape correlation with $B_2^p$
& Matched-loss CV \\
\midrule
$B_2^p$
& ---
& $0.0367$
\\
Random $8$-dimensional subspace
& $0.99987\pm0.00010$
& $0.0325$
\\
$256$-sparse unit probe
& $0.99385\pm0.00641$
& $0.0261$
\\
Normalized ellipsoid
& $0.99991\pm0.00006$
& $0.0470$
\\
\bottomrule
\end{tabular}
\end{table}

All three structured probes have temporal trajectories that are highly
correlated with the Euclidean-ball trajectory. The random subspace and
ellipsoid are nearly indistinguishable from the ball at the level of
the standardized temporal statistic. The sparse probe has a slightly
lower and more variable correlation, but its mean remains above
$0.99$. It therefore does not provide a clear probe-dependent failure
mode in the experiment.

Matched-loss variability across initializations is also moderate:
the average coefficient of variation ranges from approximately
$2.6\%$ to $4.7\%$ across the four probes. These values are larger than
the initialization variability in the controlled logistic problem, but
the fixed-probe trajectories remain reproducible at the scale relevant
to the comparisons in this section.

The strong shape correlations are qualitatively consistent with the
fixed-probe picture of Section~\ref{sec:trace-shape}: several probes
can share much of the same temporal evolution when the underlying
normalized Fisher geometry changes coherently. However,
Theorem~\ref{thm:probe-stability} is not being applied here as a
numerical bound. The present experiment uses a diagonal approximation
to a neural-network model Fisher, and the correlation statistic in
\eqref{eq:mlp-probe-shape-correlation} is only a descriptive diagnostic.
In particular, the experiment does not establish uniform stability
over arbitrary, data-dependent, or adaptively selected probe classes.

\begin{figure}[t]
    \centering
    \includegraphics[width=\textwidth]{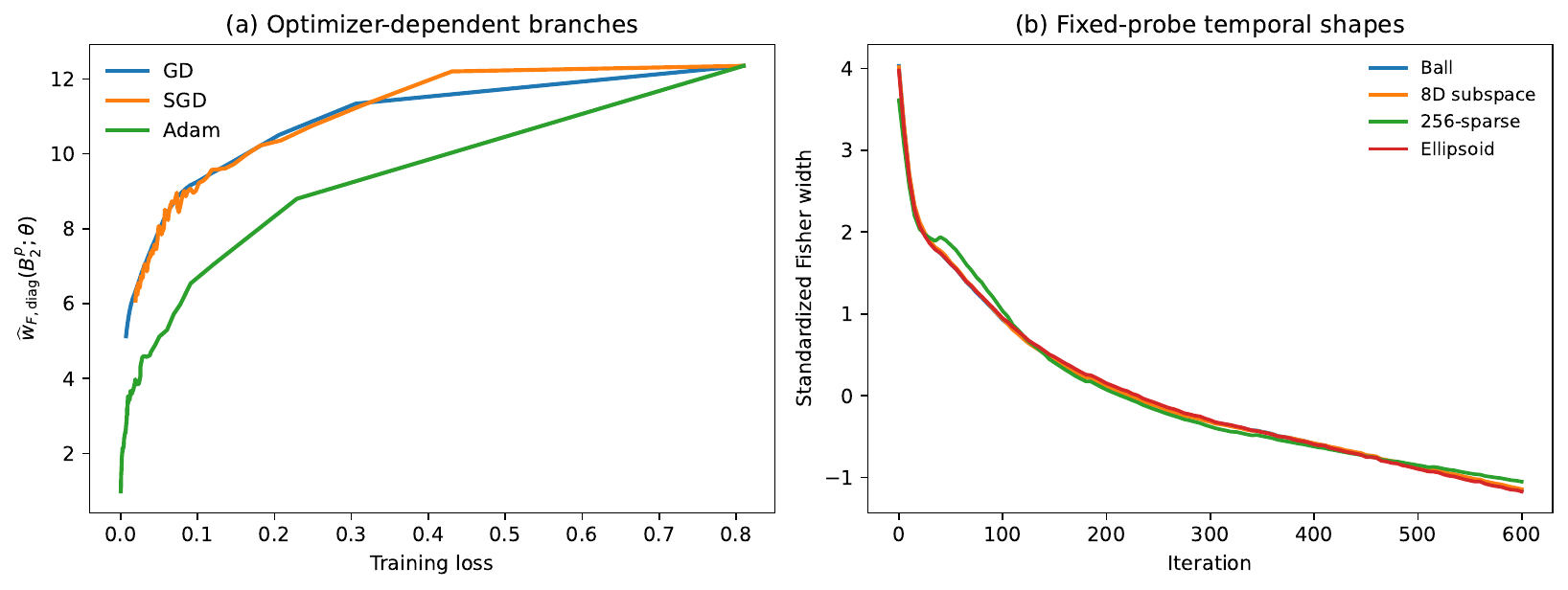}
    \caption{
    Nonlinear MLP diagnostics under the diagonal model-Fisher approximation.
    Left: GD and SGD remain close at matched loss, whereas Adam follows a
    visibly displaced branch. Right: standardized Fisher widths for four
    fixed probes along the GD trajectory retain closely similar temporal
    shapes. Numerical summaries over all six initializations are reported
    in Tables~\ref{tab:optimizer-branch-comparison}
    and~\ref{tab:mlp-probe-dependence}.
    }
    \label{fig:mlp-boundary}
\end{figure}

Figure~\ref{fig:mlp-boundary} summarizes these two empirical effects:
optimizer dependence is pronounced for Adam, whereas the fixed probes
tested retain highly similar temporal shapes.

\subsection{Scope of the Nonlinear Experiment}
\label{sec:mlp-scope}

The MLP experiment is an empirical boundary test rather than an
extension of the population theory. Within the tested diagonal
model-Fisher geometry, two conclusions are visible. First, the
matched-loss description is optimizer-dependent: GD and SGD remain
close, while Adam follows a substantially displaced branch. Second,
this optimizer dependence is not accompanied by strong temporal
disagreement among the fixed probes considered here; all three
structured probes track the Euclidean-ball trajectory closely.

These two observations should not be conflated. Probe-shape agreement
concerns the evolution of several fixed Gaussian-width observables
under one approximate Fisher geometry, whereas the Adam displacement
concerns the relation between Fisher width and loss across optimization
trajectories. Neither result establishes the corresponding behaviour
for the full neural-network Fisher matrix.

The nonlinear experiment therefore supports a branchwise rather than
universal interpretation of the empirical phenomenon. Some
loss-parametrized organization persists after leaving the controlled
Gaussian setting, but the branch itself can depend substantially on
the optimizer. At the same time, the fixed probes tested here retain a
high degree of temporal coherence. These observations define the
empirical scope of the present study rather than a general law for
neural-network training.

\section{Discussion and Conclusion}
\label{sec:discussion}

The results support a branchwise interpretation of loss-parametrized
Fisher width rather than a universal loss-only law. Loss alone does not
determine Fisher geometry throughout parameter space, even in the
controlled Gaussian--logistic model. What can emerge instead is a
distinguished branch selected by the geometry of the problem and by the
learning dynamics. The theory and experiments clarify this statement at
three different levels: what is proved exactly in the controlled model,
what is supported empirically beyond the exact theory, and what remains
open.

\paragraph{What is proved in the controlled model.}

In the Gaussian-teacher logistic population model, the Fisher matrix has
an exact radial one-spike structure, while the loss depends separately
on parameter norm and teacher alignment. This separation leads to the
global misalignment bound
\[
    L(\alpha,\beta)-L_\star(r)
    \geq
    \frac12 b(r)\beta^2,
    \qquad
    r=\sqrt{\alpha^2+\beta^2}.
\]
Consequently, on every loss level below $\log 2$, the teacher-aligned
state has the smallest parameter norm and maximizes both the Fisher
trace and the Euclidean-ball Fisher width. The aligned
loss-parametrized relation is therefore an upper envelope on the
level set, not an identity shared by all parameters with the same loss.
In particular, equal loss can correspond to different Fisher widths.

The population dynamics explain why this extremal branch is relevant
to learning. Under gradient flow,
\[
    \alpha(t)\to\infty,
    \qquad
    \beta(t)\to0,
    \qquad
    \frac{\alpha(t)}{\|\theta(t)\|_2}\to1.
\]
More precisely,
\[
    \frac{\alpha(t)^3}{t}
    \longrightarrow
    \frac{\pi^2}{2\sqrt{2\pi}},
\]
and, when $\beta(0)\neq0$,
\[
    \frac{\log|\beta(t)|}{\alpha(t)^2}
    \longrightarrow
    -\frac3{\pi^2}.
\]
Thus population gradient flow asymptotically selects the aligned
envelope rather than merely remaining near a generic low-loss region.

Combining this dynamical selection with the aligned low-loss Fisher
asymptotics yields, for $d\geq2$,
\begin{equation}
    \frac{
        w_F(B_2^d;\theta(t))
    }{
        \sqrt{L(\theta(t))}
    }
    \longrightarrow
    c_d,
    \qquad
    c_d
    =
    \frac{\sqrt6}{\pi}\,
    \mathbb E[\chi_{d-1}].
    \label{eq:discussion-dynamic-eos}
\end{equation}
The relation therefore holds asymptotically along the actual
population gradient-flow trajectory, rather than only after restricting
the parameter to the aligned ray.

The near-aligned analysis describes finite deviations from this branch.
For a general fixed compact probe, the matched-loss Fisher-width
residual is bounded by $O(|\beta|)$ on compact aligned-coordinate
intervals. For orthogonally invariant probes, the orientation effect
cancels and the bound improves to $O(\beta^2)$. The linear order for a
general probe is an upper bound and is not claimed to be sharp.

A separate deterministic result isolates a more general geometric
mechanism:
\[
    w_F(T;\theta)
    =
    \sqrt{\operatorname{Tr}G(\theta)}\,
    w\!\left(\overline G(\theta)^{1/2}T\right).
\]
For any fixed compact probe, perturbation of the normalized
square-root Fisher matrix controls the normalized width through
\[
    \chi_T
    =
    \frac{\gamma_dR_T}{a_T}.
\]
This factorization does not relate Fisher trace to loss and therefore
does not itself produce a loss-parametrized branch. It only separates a
common Fisher scale from probe-dependent normalized geometry.

The finite-sample theory has a similarly specific scope. For an
independent Gaussian evaluation sample, the empirical Fisher trace
admits a pointwise approximation whose error depends on the effective
sample size $nm(r)$. Conditional on a trained trajectory, the same
control extends simultaneously to any finite collection of recorded
checkpoints by a union bound. No uniform concentration result over a
continuous data-dependent trajectory is claimed.

\paragraph{What is supported empirically.}

The numerical experiments test these pre-existing theoretical
predictions rather than infer them from the data. In the controlled
logistic experiment, GD, damped NGD, noisy GD, and Adam produce closely
agreeing full-Fisher-width curves when compared over their shared
observed loss intervals. The same matched-loss quantity is highly
stable across the small random initializations considered.

The population calculations independently check the structural
predictions of the Gaussian theory. Equal-loss non-aligned states lie
below the aligned Fisher-width envelope, and the quantitative
misalignment ratio approaches the non-unit local value predicted by
the near-aligned expansion. Numerical integration of the exact
population flow is also consistent with the asymptotic constants for
the aligned and orthogonal coordinates, while the normalized
Fisher-width ratio approaches the limit in
\eqref{eq:discussion-dynamic-eos}. These calculations serve as
validation of the analytic results, not as estimates from which the
theory is derived.

The nonlinear experiment gives a different picture. For the
two-hidden-layer MLP, the full Fisher matrix is not formed; the
observable is a Gaussian width under a diagonal approximation to the
model Fisher. Within this approximation, GD and mini-batch SGD remain
close at matched loss, whereas Adam follows a substantially displaced
branch. Thus loss reparametrization does not remove optimizer dependence
in this setting.

At the same time, the fixed probes tested in the MLP exhibit very
similar temporal shapes. The Euclidean ball, random subspace, sparse
unit probe, and normalized ellipsoid remain strongly correlated over
training, and their matched-loss variability across initializations is
moderate. In particular, the experiment does not support the
earlier interpretation of the sparse probe as a clear failure mode.
Within the tested family, optimizer dependence is more pronounced than
probe dependence.

These observations should remain separate from the exact theory. The
MLP results concern one architecture, one training protocol, and a
diagonal model-Fisher approximation. They do not establish the same
behaviour for the full Fisher matrix, arbitrary neural networks, or
arbitrary probe classes.

\paragraph{What remains open.}

Several questions are left unresolved by the present analysis. The
global envelope and dynamical-selection arguments rely strongly on
isotropic Gaussian covariates and a noiseless linear teacher. It remains
to determine which parts survive under anisotropic covariates, label
noise, misspecification, or more general population models. In
particular, the radial Fisher structure used here is special and need
not persist once rotational symmetry is lost.

At the finite-sample level, the current theory controls independent
evaluation at fixed parameters or finitely many recorded checkpoints.
A trajectory-uniform result would require concentration along an
adaptive path and would have to account explicitly for the decay of the
effective Fisher sample size in the low-loss regime.

For neural networks, the main unresolved question is whether the
branchwise behaviour observed under the diagonal model Fisher persists
under richer approximations or the full Fisher geometry. Structured
metrics such as block or Kronecker-factored approximations would provide
an intermediate setting between the tractable diagonal experiment and
the full matrix. The strong temporal agreement among the fixed probes
considered here also leaves open the behaviour of more localized,
data-dependent, or adaptively selected probes, which are outside the
fixed-probe stability theorem.

The term \emph{equation of state} is therefore used only as shorthand
for a reproducible relation between aggregate observables along a
specified family of trajectories. The main conclusion is not that loss
globally determines Fisher width. Rather, the controlled model shows a
more limited mechanism:
\[
    \text{level-set envelope}
    \quad\longrightarrow\quad
    \text{dynamical branch selection}
    \quad\longrightarrow\quad
    \text{loss-parametrized Fisher-width asymptotics}.
\]
Beyond that model, the experiments indicate that part of this
organization can persist, but the selected branch can depend on the
optimizer and on how Fisher geometry is represented. The resulting
picture is therefore branchwise rather than universal.

\section*{Acknowledgments}

The author acknowledges the use of ChatGPT and Claude for editorial
assistance and limited support in preparing numerical experiments. The
author independently checked the mathematical arguments, experimental
designs, numerical results, and source code, and takes responsibility
for the final content of the manuscript.

\paragraph{Code availability.}
Code for the controlled logistic experiments, population-flow
calculations, and MLP diagonal model-Fisher experiments is publicly
available at \\
\texttt{https://github.com/vukhacky/fisher-width-dynamics}. \\
The repository includes the
random seeds and configurations used to reproduce the reported
numerical results.

\bibliographystyle{plainnat}
\bibliography{references}

\appendix
\section{Technical Proofs}
\label{app:technical-proofs}

This appendix provides the technical proofs deferred from the main
text: the Gaussian-width perturbation bound used in
Section~\ref{sec:trace-shape} and the finite-sample Fisher-trace
concentration result of Section~\ref{sec:finite-sample}. The population
results of Section~\ref{sec:gaussian-theory} are proved in the main
text.

\subsection{Proof of the Probe-Stability Bound}
\label{app:trace-shape-proofs}

\begin{proof}[Proof of Theorem~\ref{thm:probe-stability}]
Let $T\subset\mathbb R^d$ be nonempty and compact, and write
\[
    R_T
    :=
    \sup_{v\in T}\|v\|_2.
\]
For fixed $g\in\mathbb R^d$ and linear maps $A,B$,
\[
\begin{aligned}
&
\left|
    \sup_{v\in T}\langle g,Av\rangle
    -
    \sup_{v\in T}\langle g,Bv\rangle
\right|
\\
&\qquad\leq
\sup_{v\in T}
\left|
    \langle g,(A-B)v\rangle
\right|
\\
&\qquad\leq
R_T\|A-B\|_{\mathrm{op}}\|g\|_2.
\end{aligned}
\]
Taking expectation over $g\sim N(0,I_d)$ gives
\begin{equation}
    |w(AT)-w(BT)|
    \leq
    \gamma_dR_T\|A-B\|_{\mathrm{op}},
    \qquad
    \gamma_d:=\mathbb E\|g\|_2.
    \label{eq:app-width-perturbation}
\end{equation}

Now write
\[
    G_t=s_t\overline G_t,
    \qquad
    s_t=\operatorname{Tr}G_t.
\]
By the trace--shape identity,
\[
    \frac{w_F(T;\theta_t)}{\sqrt{s_t}}
    =
    w\!\left(\overline G_t^{1/2}T\right).
\]
Applying \eqref{eq:app-width-perturbation} with
\[
    A=\overline G_t^{1/2},
    \qquad
    B=\overline G_\star^{1/2},
\]
and
\[
    a_T
    :=
    w\!\left(\overline G_\star^{1/2}T\right)
\]
yields
\[
    \left|
        \frac{w_F(T;\theta_t)}{\sqrt{s_t}}
        -
        a_T
    \right|
    \leq
    \gamma_dR_T
    \left\|
        \overline G_t^{1/2}
        -
        \overline G_\star^{1/2}
    \right\|_{\mathrm{op}}.
\]
This is \eqref{eq:absolute-normalized-probe-bound}. If $a_T>0$,
division by $a_T$ gives \eqref{eq:relative-probe-bound}.
\end{proof}

\subsection{Details for the Finite-Sample Trace Bound}
\label{app:finite-sample-details}

We prove Proposition~\ref{prop:pointwise-trace-concentration}.
Fix $\theta\in\mathbb R^d$ independently of the evaluation sample.
By rotational invariance, take
\[
    \theta=re_1
\]
and write
\[
    X_i=(Z_i,U_i),
    \qquad
    Z_i\sim N(0,1),
    \qquad
    U_i\sim N(0,I_{d-1}),
\]
with $Z_i$ and $U_i$ independent. Set
\[
    v_i:=v(rZ_i),
    \qquad
    m:=m(r)=\mathbb E[v(rZ)].
\]

From \eqref{eq:trace-variance-residual},
\begin{equation}
\begin{aligned}
&
\operatorname{Tr}\widehat G_n(\theta)
-
d\,\widehat m_n(\theta)
\\
&\qquad=
\frac1n
\sum_{i=1}^n
\left(
    Q_i+S_i
\right),
\end{aligned}
\label{eq:app-trace-decomposition}
\end{equation}
where
\[
    Q_i
    :=
    v_i(Z_i^2-1)
\]
and
\[
    S_i
    :=
    v_i
    \left(
        \|U_i\|_2^2-(d-1)
    \right).
\]
Since $\mathbb E[S_i]=0$,
\[
\begin{aligned}
    \mathbb E[Q_i+S_i]
    &=
    \mathbb E[v(rZ)(Z^2-1)]
    \\
    &=
    a(r)-b(r)
    \\
    &=
    m(r)(\kappa(r)-1),
\end{aligned}
\]
which proves \eqref{eq:trace-bias}.

\paragraph{Longitudinal fluctuation.}

Since $0\leq v\leq1/4$,
\[
    |Q_i|
    \leq
    \frac14|Z_i^2-1|,
\]
so $Q_i$ has a uniformly bounded sub-exponential norm. Moreover,
using $v^2\leq v/4$,
\[
\begin{aligned}
    \operatorname{Var}(Q_i)
    &\leq
    \mathbb E[
        v(rZ)^2(Z^2-1)^2
    ]
    \\
    &\leq
    \frac14
    \mathbb E[
        v(rZ)(Z^2-1)^2
    ].
\end{aligned}
\]

Since $Z^4$ is increasing in $|Z|$, whereas $v(rZ)$ is decreasing in
$|Z|$, the covariance inequality for oppositely monotone functions of
$|Z|$ gives
\[
    \mathbb E[v(rZ)Z^4]
    \leq
    \mathbb E[v(rZ)]\mathbb E[Z^4]
    =
    3m.
\]
Hence
\[
\begin{aligned}
    \mathbb E[
        v(rZ)(Z^2-1)^2
    ]
    &=
    \mathbb E[v(rZ)Z^4]
    -
    2a(r)
    +
    m
    \\
    &\leq
    4m,
\end{aligned}
\]
and therefore
\[
    \operatorname{Var}(Q_i)\leq m.
\]

Bernstein's inequality gives, with probability at least
$1-\delta/2$,
\begin{equation}
    \left|
        \frac1n
        \sum_{i=1}^n
        \left(
            Q_i-\mathbb E Q_i
        \right)
    \right|
    \leq
    C
    \left[
        \sqrt{
            \frac{
                m\log(4/\delta)
            }{n}
        }
        +
        \frac{\log(4/\delta)}{n}
    \right].
    \label{eq:app-longitudinal-bound}
\end{equation}

\paragraph{Transverse fluctuation.}

Conditionally on $Z_1,\ldots,Z_n$,
\[
    S_i
    =
    \sum_{j=1}^{d-1}
    v_i(U_{ij}^2-1).
\]
The variables
\[
    \left\{
        v_i(U_{ij}^2-1)
    \right\}_{i,j}
\]
are conditionally independent and centered. Since
\[
    \|U_{ij}^2-1\|_{\psi_1}\leq C
\]
and $0\leq v_i\leq1/4$, they also have uniformly bounded conditional
sub-exponential norms. Their conditional variance proxy satisfies
\[
    C(d-1)\sum_{i=1}^n v_i^2
    \leq
    Cd\sum_{i=1}^n v_i,
\]
again using $v_i^2\leq v_i/4$.

Define the event
\[
    \mathcal A
    :=
    \left\{
        \frac1n
        \sum_{i=1}^n v_i
        \leq
        2m
    \right\}.
\]
Since $0\leq v_i\leq1/4$ and
\[
    \operatorname{Var}(v_i)
    \leq
    \mathbb E[v_i^2]
    \leq
    \frac{m}{4},
\]
Bernstein's inequality gives
\begin{equation}
    \mathbb P(\mathcal A^c)
    \leq
    e^{-cnm}
    \label{eq:app-weight-event}
\end{equation}
after adjusting the universal constant $c$.

On $\mathcal A$, conditional Bernstein gives, with conditional
probability at least $1-\delta/2$,
\begin{equation}
    \left|
        \frac1n
        \sum_{i=1}^nS_i
    \right|
    \leq
    C
    \left[
        \sqrt{
            \frac{
                dm\log(4/\delta)
            }{n}
        }
        +
        \frac{\log(4/\delta)}{n}
    \right].
    \label{eq:app-transverse-bound}
\end{equation}

Combining
\eqref{eq:app-longitudinal-bound} and
\eqref{eq:app-transverse-bound}, and enlarging the universal constants,
gives
\begin{align}
&
\left|
    \operatorname{Tr}\widehat G_n(\theta)
    -
    d\,\widehat m_n(\theta)
    -
    m(r)(\kappa(r)-1)
\right|
\nonumber\\
&\qquad\leq
C
\left[
    \sqrt{
        \frac{
            d\,m(r)\log(4/\delta)
        }{n}
    }
    +
    \frac{\log(4/\delta)}{n}
\right]
\label{eq:app-absolute-trace-bound}
\end{align}
with probability at least
\[
    1-\delta-2e^{-cnm(r)}.
\]
This proves the absolute bound in
Proposition~\ref{prop:pointwise-trace-concentration}.

\paragraph{Relative bound.}

A further Bernstein inequality gives
\begin{equation}
    \mathbb P\left(
        \widehat m_n(\theta)
        <
        \frac12m(r)
    \right)
    \leq
    2e^{-cnm(r)}.
    \label{eq:app-m-lower-tail}
\end{equation}
On the complement of this event,
\[
    d\,\widehat m_n(\theta)
    \geq
    \frac12dm(r).
\]
Since
\[
    |m(r)(\kappa(r)-1)|
    \leq
    m(r),
\]
division of \eqref{eq:app-absolute-trace-bound} by
$d\,\widehat m_n(\theta)$ gives
\begin{align}
\left|
    \frac{
        \operatorname{Tr}\widehat G_n(\theta)
    }{
        d\,\widehat m_n(\theta)
    }
    -1
\right|
\leq
C\left[
    \frac1d
    +
    \sqrt{
        \frac{
            \log(4/\delta)
        }{
            ndm(r)
        }
    }
    +
    \frac{
        \log(4/\delta)
    }{
        ndm(r)
    }
\right].
\label{eq:app-relative-trace-bound}
\end{align}
Combining the failure probabilities gives probability at least
\[
    1-\delta-4e^{-cnm(r)}.
\]
This proves \eqref{eq:relative-trace-concentration}.

\paragraph{Finite checkpoints.}

For Corollary~\ref{cor:finite-checkpoint-control}, condition on the
training sample $\mathcal D_{\mathrm{tr}}$. The recorded parameters
\[
    \theta_{t_1},\ldots,\theta_{t_N}
\]
are then fixed and independent of the evaluation sample. Apply the
absolute pointwise bound to checkpoint $j$ with failure probability
$\delta/N$. Since
\[
    m(\|\theta_{t_j}\|_2)
    \geq
    m_{\min},
\]
a union bound gives simultaneous control of all $N$ checkpoints with
probability at least
\[
    1-\delta
    -
    2N e^{-c n_{\mathrm{ev}}m_{\min}},
\]
and replaces $\log(4/\delta)$ by
\[
    \log\left(\frac{4N}{\delta}\right).
\]
This is \eqref{eq:finite-checkpoint-bound}.
\end{document}